\documentclass[twoside,11pt]{article}

\usepackage{jmlr2e}
\usepackage{algorithm}
\usepackage{algorithmic}
\usepackage{amsmath}

\newtheorem{lem}{Lemma}

 \newtheorem{property}{Property}

    \def\independenT#1#2{\mathrel{\setbox0\hbox{$#1#2$}%
    \copy0\kern-\wd0\mkern4mu\box0}}

\ShortHeadings{Stable Graphical Models}{Misra and Kuruoglu}
\firstpageno{1}

\begin{document}

\title{Stable Graphical Models}

\author{\name{Navodit Misra} \email{misra@molgen.mpg.de}\\
       \addr  Max Planck Institute for Molecular Genetics\\
       \addr Ihnestr. 63-73, 14195 Berlin, Germany
        \AND
      \name{Ercan E. Kuruoglu} \email{ercan.kuruoglu@isti.cnr.it}\\
        \addr ISTI-CNR\\
        \addr Via G. Moruzzi 1, 56124 Pisa, Italy\\ and\\
        \addr Max Planck Institute for Molecular Genetics\\
        \addr Ihnestr. 63-73, 14195 Berlin, Germany
    }
\editor{}
\maketitle

\begin{abstract}
Stable random variables are motivated by the central limit theorem for densities with (potentially) unbounded variance and can be thought of as natural generalizations of the Gaussian distribution to skewed and heavy-tailed phenomenon. In this paper, we introduce $\alpha$-stable graphical ($\alpha$-SG) models, a class of multivariate stable densities that can also be represented as Bayesian networks whose edges encode linear dependencies between random variables. One major hurdle to the extensive use of stable distributions is the lack of a closed-form analytical expression for their densities. This makes penalized maximum-likelihood based learning computationally demanding. We establish theoretically that the {\it Bayesian information criterion} (BIC) can asymptotically be reduced to the computationally more tractable {\it minimum dispersion criterion} (MDC) and develop {\tt StabLe}, a structure learning algorithm based on MDC. We use simulated datasets for five benchmark network topologies to empirically demonstrate how {\tt StabLe} improves upon ordinary least squares (OLS) regression. We also apply {\tt StabLe} to microarray gene expression data for lymphoblastoid cells from 727 individuals belonging to eight global population groups. We establish that {\tt StabLe} improves test set performance relative to OLS via ten-fold cross-validation. Finally, we develop {\tt SGEX}, a method for quantifying differential expression of genes between different population groups.
\end{abstract}


\section{Introduction}

Stable distributions have found applications in modeling several real-life phenomena~\citep{ Berger63,Mandel63, Nikias, gallardo2000use, achim2001novel} and have robust theoretical justification in the form of the generalized central limit theorem~\citep{feller1968introduction,Nikias, Nolan13}. Several special instances of multivariate generalization of stable distributions have also been described in literature~\citep{Samorodnitsky94, NolanMulti}. Multivariate stable densities have previously been applied to modeling  wavelet coefficients with bivariate $\alpha$-stable distributions~\citep{achim2005image}, inferring parameters for linear models of network flows~\citep{Bickson} and stock market fluctuations~\citep{bonato2012modeling}.

In this paper, we describe $\alpha$-stable graphical ($\alpha$-SG) models, a new class of multivariate stable densities that can be represented as directed acyclic graphs (DAG) with arbitrary network topologies. We prove that these multivariate densities also correspond to linear regression-based Bayesian networks and establish a model selection criterion that is asymptotically equivalent to the {\it Bayesian information criterion} (BIC). Using simulated data for five benchmark network topologies, we empirically show how $\alpha$-SG models improve structure and parameter learning performance for linear regression networks with additive heavy-tailed noise.
 
One motivation for the present work comes from potential applications to computational biology, especially in genomics, where Bayesian network models of gene expression profiles are a popular tool ~\citep{friedman2000using, ben2000tissue, friedman2004inferring}. A common approach to network models of gene expression involves learning linear regression-based Gaussian graphical models. However, the distribution of experimental microarray intensities shows a clear skew and may not necessarily be best described by a Gaussian density (Section~\ref{GeneEx}). Another aspect of microarray intensities is that they represent the average mRNA concentration in a population of cells. Assuming the number of mRNA transcripts within each cell to be independent and identically distributed, the generalized central limit theorem suggests that the observed shape should asymptotically (for large population size) approach a stable density~\citep{feller1968introduction,Nikias,Nolan13}. Univariate stable distributions have previously been used to model gene expression data~\citep{diego2009modelling, salas2009heavy} and it is therefore natural to consider multivariate $\alpha$-stable densities as models for mRNA expression for larger sets of genes. In Section~\ref{GeneEx} we provide empirical evidence to support this reasoning. We further develop $\alpha$-stable graphical ($\alpha$-SG) models for  quantifying differential expression of genes from microarray data belonging to phase III of the HapMap project~\citep{HapMap3,montgomery2010transcriptome,stranger2012patterns}.

The rest of the paper is structured as follows : Section~\ref{intro} describes the basic notation and background concepts for Bayesian networks and stable densities. Section~\ref{SGmodels} introduces $\alpha$-SG models and establishes that these models are Bayesian networks that also represent multivariate stable distributions with finite spectral measures. Section~\ref{Learning} establishes the equivalence of the popular but (in this case) computationally challenging {\it Bayesian information criterion} (BIC) for structure learning and the computationally more tractable {\it minimum dispersion criterion} (MDC), for all $\alpha$-SG models that represent symmetric densities. Furthermore, we establish how data samples from any $\alpha$-SG model can be combined to generate samples from a partner symmetric $\alpha$-SG model with identical network topology and regression coefficients.  Using these theoretical results we design {\tt StabLe}, an efficient algorithm  that combines ordering-based search (OBS)~\citep{Teyssier} for structure learning with the iteratively re-weighted least squares (IRLS) algorithm~\citep{byrd1979convergence} for learning the regression parameters via least $l_p$ norm estimation. Finally, in Section~\ref{Validation} we implement the structure and parameter learning algorithm on simulated and expression microarray data sets.

\section{Methods}
In this section we develop the theory and algorithms for learning $\alpha$-SG models from data. First, we discuss some well-established results for Bayesian networks and $\alpha$-stable densities.
\subsection{Background}\label{intro}

We begin with an introduction to Bayesian network models~\citep{Pearl} for the joint probability distribution of a finite set of random variables $\mathcal{X}= \{X_1,\ldots X_m\}$. A Bayesian network $B(G,\Theta)$ is specified by a directed acyclic graph (DAG) $G$, whose vertices represent random variables in $\mathcal{X}$ and a set  of parameters $\Theta=\{\theta_i| X_i\in \mathcal{X}\}$, that determine the conditional probability distribution $p(X_i|Pa(X_i),\theta_i)$ for each variable $X_i\in\mathcal{X}$ given the state of its parents $Pa(X_i)\subseteq \mathcal{X}\setminus\{X_i\}$ in $G$~\citep{Koller}. We will overload the symbols $X_j$ and $Pa(X_j)$ to represent both sets of random variables and their instantiations. The directed acyclic graph $G$ implies a factorization of the joint probability density into terms representing  each variable $X_i$ and its parents  $Pa(X_i)$ (called a {\it family}) such that :

\begin{equation}
P_B(\mathcal{X})=\prod_{i=1}^{|\mathcal{X}|}p(X_i | Pa(X_i), \theta_i)
\end{equation}

The dependence of $p(X_i | Pa(X_i),\theta_i)$ on $\theta_i$  is usually specified by an appropriately chosen  family of parametrized probability densities for the random variables, such as Gaussian or $\log$-Normal. In this paper, we will use multivariate stable densities to model the random variables in $\mathcal{X}$. The primary motivation for modeling continuous random variables using stable distributions comes from the  generalization of the central limit theorem to distributions with unbounded variance~\citep{feller1968introduction, Nikias}. In the limit of large $N$, all sums of $N$ independent, identically distributed random variables approach a stable density. A formal definition for stable random variables can be provided in terms of the characteristic function (Fourier transform of the density function)

\begin{definition}\label{StableDef} A {\it stable random variable} $X\sim~S_{\alpha}( \beta, \gamma, \mu)$, is defined for each $\alpha\in(0,2]$, $\beta\in[-1,1]$, $\gamma\in(0,\infty)$ and $\mu\in(-\infty,\infty)$. The probability density $f(X|\alpha,\beta,\gamma,\mu)$ is implicitly specified by a characteristic function $\phi(q|\alpha,\beta,\gamma,\mu)$ :
\begin{eqnarray}
\phi(q|\alpha,\beta,\gamma,\mu) &\equiv& \mathbb{E}[\exp(\imath qX)]\nonumber \\
&=&\int_{-\infty}^{\infty} f(X|\alpha,\beta,\gamma,\mu)\exp(\imath qX) dX\nonumber\\
&=&\exp\bigl(\imath\mu q -\gamma|q|^\alpha [1-\imath\beta~\mathrm{sign}(q) r(q,\alpha)]\bigr) \nonumber\\
 \mathrm{ where,} ~r(q,\alpha)&=&\left\{\begin{array}{cc} \tan\frac{\alpha\pi}{2} & ~\alpha\neq1 \\-\frac{2}{\pi}\log|q| & ~\alpha=1\end{array}\right.\nonumber
\end{eqnarray}
\end{definition}
The parameters $\alpha, \beta, \gamma$ and $\mu$ will be called the characteristic exponent, skew, dispersion and location respectively. Unfortunately, the density $f(X|\alpha,\beta,\gamma,\mu)$ does not have a closed-form analytical expression except for the three well-known stable distributions (Figure~\ref{fig: Sdensity} and Table~\ref{LCN}).

Except for the Gaussian case, the asymptotic (large $x$) behavior of univariate $\alpha$-stable densities shows Pareto or power law tails~\citep{levy1925calcul}. The following lemma formalizes this observation~\citep{Samorodnitsky94, Nolan13}
\begin{lem}\label{StableTails}If $X\sim S_{\alpha}(\beta,\gamma,0)$ with $0<\alpha<2$, then as $x\rightarrow\infty$
\begin{eqnarray}
Pr(X>x)&\sim& (1+\beta)\gamma C_\alpha x^{-\alpha}\nonumber\\
C_\alpha=(2\int{_0}{^\infty} x^{-\alpha} \sin x dx)^{-1}&=&\frac{1}{\pi}\Gamma(\alpha)\sin(\frac{\alpha\pi}{2})\nonumber
\end{eqnarray}
\end{lem}

\begin{figure}[htb!]
\begin{center}
\includegraphics[scale=.4, angle=-90]{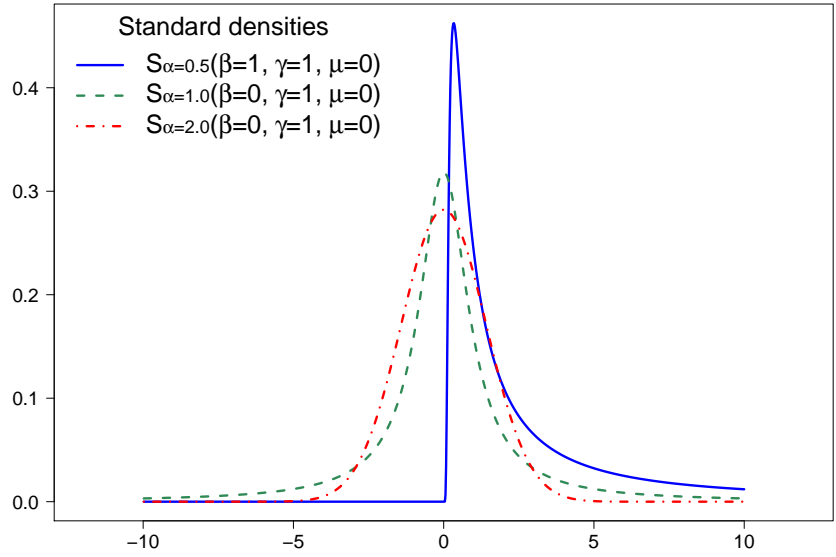}
\end{center}
\caption{The three instances of analytically known univariate $\alpha$-stable densities $S_\alpha(\beta,\gamma,\mu)$. L\'evy$(\gamma,\mu) \sim S_{0.5}(1, \gamma, \mu)$ (solid blue curves), Cauchy$(\gamma,\mu)\sim S_{1.0}(0,\gamma,\mu)$ (dashed green curves) and Normal$(\mu, \sigma)\sim S_{2.0}(0,\frac{\sigma^2}{2},\mu)$ (dot-dashed red curves). }\label{fig: Sdensity}
\end{figure}

\begin{table}
\begin{center}
\begin{tabular}{cccc}\hline Distribution & $S_\alpha(\beta,\gamma,\mu)$ & $f(X|\alpha,\beta,\gamma,\mu)$ & Support \\ \hline\textrm{L\'evy}$(\gamma,\mu)$ & $S_{0.5}(1,\gamma,\mu) $ & $\frac{\gamma}{\sqrt{2\pi}}\frac{1}{(x-\mu)^{3/2}}\exp\bigl(-\frac{\gamma^2}{2(x-\mu)}\bigr)$ & $\mu<x<\infty$ \\\textrm{Cauchy}$(\gamma,\mu) $&$ S_{1.0}(0,\gamma,\mu) $ & $\frac{1}{\pi}\frac{\gamma}{\gamma^{2}+(x-\mu)^2}$ &$ -\infty<x<\infty$ \\\textrm{Normal}$(\mu,\sigma)$ & $S_{2.0}(0,\gamma=\frac{\sigma^2}{2},\mu)$ &  $\frac{1}{2\sqrt{\pi\gamma}}\exp\bigl(-\frac{(x-\mu)^2}{4\gamma}\bigr)$ &$ -\infty<x<\infty$ \\ \hline
\end{tabular}\caption{Closed-form analytical expressions for L\'evy, Cauchy and Normal densities and the corresponding $\alpha$-stable parameters.}\label{LCN}
\end{center}
\end{table}
A word on the notation used throughout this paper. We will use the symbol $\|Y\|_p=(\sum_{\lambda} |Y_\lambda|^p)^{1/p}$ to represent the $l_p$ norm of a vector. The $l_p$ norm of a vector representing $N$ instantiations of a random variable $Z$ is related to the $p^{th}$ moment  $E(|Z|^p) = \|Z\|_p^{p}/N$. For heavy-tailed $\alpha$-stable densities, one convenient method for parameter estimation is via {\it fractional lower order moments} (FLOM) for $p<\alpha$~\citep{hardin1984skewed, Nikias}. Later, we will discuss FLOM-based parameter learning in greater detail (Section~\ref{ParamL}).

\subsection{$\alpha$-Stable Graphical Models}\label{SGmodels}

We can now introduce Bayesian network models reconstructed from stable densities that have compact representations for the characteristic function. Univariate $\alpha$-stable densities can be generalized to represent multivariate stable distributions that are defined as follows~\citep{Samorodnitsky94}, 

\begin{definition}\label{MultiDef} A $d$-dimensional {\it multivariate stable distribution} over $\mathcal{X}=\{X_1,\ldots X_d\}$ is defined by an $\alpha\in(0,2]$, $\mu\in \mathbb{R}^d$ and a spectral measure $\Lambda$ over the $d$-dimensional unit sphere $S_d$, such that the characteristic function
\begin{eqnarray}
\Phi(q|\alpha,\mu,\Lambda)&\equiv&\mathbb{E}[\exp(\imath q^T \mathcal{X})] \nonumber\\
&=& \exp\Big(-\int_{S_{d}}\psi(s^T q|\alpha)\Lambda(ds) + \imath\mu^T q\Big) \nonumber\\
\mathrm{where, } ~\psi(u|\alpha)&=&|u|^\alpha(1-\imath~\mathrm{sign}(u)r(u,\alpha))\nonumber
 \end{eqnarray} 
\end{definition}
\begin{definition}\label{ASGdef} An {\it $\alpha$-stable graphical} ($\alpha$-SG) model  $B(G,\Theta)$
is a probability distribution over $\mathcal{X}$ such that
\begin{eqnarray}
&1.&Z_{j}\equiv X_j - \sum_{X_k\in Pa(X_j)} w_{jk}X_k\sim S_{\alpha}(\beta_{j},\gamma_{j},\mu_{j})\nonumber\\
&2.& Z_j ~\mathrm{is~independent~of}~ Z_k~, ~\mathrm{if }~ Z_j\neq Z_k,~\forall X_j\in\mathcal{X}\nonumber
\end{eqnarray}
where $Pa(X_j)\subseteq \mathcal{X}\setminus \{X_j\}$ are the parent nodes of $X_j$ in the directed acyclic graph $G$ and $\Theta$ describes the distribution parameters 
\begin{eqnarray}
 w_{jk}\in\mathbb{R},&&W_j=\{w_{jk}| X_k\in Pa(X_j)\},\nonumber\\
 \theta_j=\{\alpha,\beta_j, \gamma_j,\mu_j\}\cup W_j,&&\Theta=\{\theta_i|X_i\in\mathcal{X}\}\nonumber
\end{eqnarray}
\end{definition}
It is straightforward to see that $B(G,\Theta)$ is indeed a Bayesian network.

\begin{lem}\label{BN} $B(G,\Theta)$ in Definition~\ref{ASGdef} represents a Bayesian network\end{lem}
\begin{proof}
Let $d=|\mathcal{X}|$. First note that every directed acyclic graph can be used to infer an ordering (not necessarily unique) on the variables in $\mathcal{X}$ such that all parents of each variable have a lower order than the variable itself. Suppose we index each variable with its order in an ordering compatible with the DAG, such that $X_i$ has order $i$. The proof rests on the fact that the transformation matrix from $\{Z_i\}$ to $\{X_i\}$ for such a graph is lower triangular, with each diagonal entry equal to 1. Since the determinant of a triangular matrix equals the product of its diagonal entries, the Jacobian  for the transformation (or the determinant of the transformation matrix), $|\frac{\partial(Z_1,\ldots Z_d)}{\partial(X_1,\ldots X_d)}|=1$. Furthermore, since the noise variables $Z_j$'s are independent of each other 
\begin{eqnarray}
P_B(Z_1, \ldots Z_d) &=& \prod_{j=1}^{d}f(Z_{j}|\alpha,\beta_{j},\gamma_{j},\mu_{j})\nonumber\\
\mathrm{ also,}~ p(X_j |Pa(X_j), \theta_j)&=&  f(Z_{j}|\alpha,\beta_{j},\gamma_{j},\mu_{j})\nonumber\\
\implies P_B(\mathcal{X})&=&P_B(Z_1, \ldots Z_d)|\frac{\partial(Z_1,\ldots Z_d)}{\partial(X_1,\ldots X_d)}|\nonumber\\
\implies P_B(\mathcal{X})&=&\prod_{j=1}^{d}p(X_j |Pa(X_j), \theta_j) |\frac{\partial(Z_1,\ldots Z_d)}{\partial(X_1,\ldots X_d)}|\nonumber\\
\implies P_B(\mathcal{X})&=&\prod_{j=1}^{d}p(X_j |Pa(X_j), \theta_j) \nonumber
\end{eqnarray}
Hence, $B(G,\Theta)$ is a Bayesian network.
\end{proof}
Before establishing the fact that an $\alpha$-SG model is a multivariate stable density in the sense of Definition~\ref{MultiDef}, we prove the following result (proof is provided in Appendix A) :

\begin{lem}\label{specprod} Every $d$-dimensional distribution with a characteristic function of the form
\begin{equation}
\Phi(q|\alpha,\tilde{\mu},\Lambda)=\prod_{k=1}^{d}\phi(c_{k}^Tq|\alpha,\beta_k,\gamma_k,\mu_k)~~\mathrm{where, }~c_k,q\in\mathbb{R}^d\nonumber
\end{equation}
represents a multivariate stable distribution with a finite spectral measure $\Lambda$.
\end{lem}

We are now in a position to establish that $\alpha$-SG models imply a multivariate stable density with a spectral measure concentrated on a finite number of points over the unit sphere.

\begin{lem}\label{asg} Every $\alpha$-SG model represents a multivariate stable distribution with a finite spectral measure of the form in Lemma~\ref{specprod}.\end{lem}
\begin{proof}
We will prove the lemma by induction. First, observe that every Bayesian network can be used to assign an ordering (not unique) such that $Pa(X_j)\subseteq \{X_1\ldots X_j-1\}$. As before, we will use such an ordering to index each random variable in $\mathcal{X}$, such that $X_{|\mathcal{X}|}$ has no descendants.
The base case of the lemma, where $|\mathcal{X}|=1$ is clearly true. Assume that the lemma is true for all Bayesian networks with $|\mathcal{X}|=m-1$. Then for any Bayesian network $B$ with $|\mathcal{X}|= m$ random variables
\begin{eqnarray}
\Phi_B(q) &\equiv&  \mathbb{E}[\exp(\imath q^T\mathcal{X})] \nonumber\\
&=& \int \prod_{j=1}^{|\mathcal{X}|}dX_j  f(Z_{j}|\alpha,\beta_{j},\gamma_{j},\mu_{j})\exp(\imath q_jX_j)\nonumber \\
&=&\int \Big[\prod_{j=1}^{m-1}dX_j   f(Z_{j}|\alpha,\beta_{j},\gamma_{j},\mu_{j})\exp(\imath q_jX_j)\Big]\int dX_m   f(Z_{m}|\alpha,\beta_{m},\gamma_{m},\mu_{m})\exp(\imath q_mX_m)\nonumber \\
&=&\int \Big[\prod_{j=1}^{m-1}dX_j   f(Z_{j}|\alpha,\beta_{j},\gamma_{j},\mu_{j})\exp(\imath \tilde{q_j}X_j)\Big]\int dZ_m   f(Z_{m}|\alpha,\beta_{m},\gamma_{m},\mu_{m})\exp(\imath q_mZ_m)\nonumber \\
&=&\Phi_{\tilde{B}}(\tilde{q})\phi(q_m|\alpha,\beta_m,\gamma_m,\mu_m)\nonumber \\
&&\mathrm{where }~\tilde{B}~\mathrm{ is ~the ~Bayes ~net ~on } ~\tilde{\mathcal{X}}=\mathcal{X}\setminus\{X_m\},\nonumber\\
&&\mathrm{and }~\tilde{q}_j=q_j + w_{mj}q_m|Pa(X_m)\cap\{X_j\}|~\forall ~ X_j\in \tilde{\mathcal{X}}\nonumber
\end{eqnarray}
\hspace{0.5cm} Since by assumption,
\begin{eqnarray}
\Phi_{\tilde{B}}(\tilde{q})&=&\prod_{k=1}^{m-1}\phi(s_k^T\tilde{q}|\alpha,\beta_k,\gamma_k,\mu_k)\nonumber\\
\implies \Phi_B(q)&=&\phi_{\tilde{B}}(\tilde{q})\phi(q_m|\alpha,\beta_m,\gamma_m,\mu_m)\nonumber\\
&=&\prod_{k=1}^{m}\phi(\tilde{s}_k^Tq|\alpha,\beta_k,\gamma_k,\mu_k),~\mathrm{ where : }\nonumber\\
\tilde{s}_{k}^Tq &=&
\left\{
\begin{array}{cc}
 \sum_{j=1}^{m-1}s_{k,j}(q_j + w_{mj}q_m|Pa(X_m)\cap\{X_j\}|)& k<m\\
q_m&k=m
\end{array}\right.\nonumber
\end{eqnarray}
Therefore, $\Phi_B(q)$ represents a $m$-dimensional multivariate stable distribution with a finite spectral measure~(Lemma~\ref{specprod}). Therefore, by induction, every $\alpha$-SG model represents a multivariate stable distribution with a finite spectral measure of the form in Lemma~\ref{specprod}.
\end{proof}

\subsection{Learning $\alpha$-SG Models}\label{Learning}
It is straight forward to use the characterization of stable random variables in Definition~\ref{StableDef} to verify the following well-known properties~\citep{Samorodnitsky94},
 
\begin{property}\label{Aggregation}If $X_1\sim S_{\alpha}(\beta_1, \gamma_1, \mu_1)$ and $X_2\sim S_{\alpha}(\beta_2, \gamma_2, \mu_2)$ are independent stable random variables, then $Y=X_1+X_2 \sim  S_{\alpha}( \beta, \gamma, \mu)$, with
\begin{eqnarray*}
\beta=\frac{\beta_1\gamma_1+\beta_2\gamma_2}{\gamma_1 +\gamma_2} ~,& \gamma=(\gamma_1+\gamma_2)~,& \mu=\mu_1+\mu_2
\end{eqnarray*}
\end{property}

\begin{property}\label{Scale}If $X \sim S_{\alpha}( \beta, \gamma, \mu)$ and $c, d \in\mathbb{R}$, then
\begin{eqnarray}
cX + d &\sim&\left\{\begin{array}{cc}S_{\alpha}\Big(\mathrm{sign}(c)\beta, |c|^{\alpha}\gamma, c\mu + d\Big) ~,& \alpha\neq1 \\S_{\alpha}\Big(\mathrm{sign}(c)\beta, |c|\gamma, c(\mu - \frac{2\gamma\beta\ln|c|}{\pi}) + d \Big) ~,&\alpha=1\end{array}\right.\nonumber
\end{eqnarray}
\end{property}

A popular method for structure learning in Bayesian network models is based on the {\it Bayesian information criterion} (BIC) which is also equivalent to the minimum description length (MDL) principle~\citep{Schwarz, Heckerman00}.

\begin{definition} Given a data set $D=\{D_1, \ldots, D_{N}\}$, the {\it Bayesian Information Score} $S_{BIC}(B|D)$ for a Bayesian network $B(G,\Theta)$ is defined as,
\begin{equation}
S_{BIC}(B|D)= \sum_{D_{\lambda}\in D}\log\big[P_B(D_\lambda)\big] -  \sum_{X_i\in \mathcal{X}}\frac{|Pa(X_i)|}{2}\log N \nonumber
\end{equation}
The {\it Bayesian information criterion (BIC)} selects the Bayesian network that maximizes this score over the space of all directed acyclic graphs $G$ and parameters $\Theta$.
\end{definition}

The major stumbling block in using stable densities is due to the fact that there is no known closed-form analytical expression for them (apart from special cases representing Gaussian, Cauchy and Levy distributions). This makes BIC based inference computationally demanding due to the marginal likelihood  term $P_B[D_\lambda]$. One main contribution of this paper is an efficient method of learning the network structure and parameters for $\alpha$-SG models. The next lemma establishes a new result that is useful in efficiently  solving the learning problem.

\begin{lem}\label{lk_gamma}
Given a data set $D_Y=\{Y_1, \ldots, Y_{N}\}$ generated from a stable random variable $Y\sim S_{\alpha}(\beta,\gamma,\mu)$
\begin{eqnarray}
\sum_{\lambda=1}^{N}\log\big[f(Y_\lambda|\alpha,\beta,\gamma,\mu)\big]&=& -N\Big(\log\gamma +  h(Y|\alpha,\beta)\Big)\nonumber\\
\mathrm{where, }\lim_{N\rightarrow\infty} h(Y|\alpha,\beta)&=&- \int dY f(Y|\alpha,\beta,1,0) \log f(Y|\alpha,\beta,1,0)\nonumber\\
&=& H\Big[S_{\alpha}(\beta,1,0)\Big]\nonumber
\end{eqnarray}
\end{lem}
\begin{proof}
Since $Y$ includes samples from a  stable distribution, $Y\sim S_{\alpha}( \beta,\gamma,\mu)$ by definition, performing a change of variable to 
\begin{eqnarray}\label{trans}
Y \rightarrow \tilde{Y}&=&\frac{Y }{\gamma^{1/\alpha}}-\tilde{\mu}\\
\mathrm{where,}~\tilde{\mu}&=&\left\{\begin{array}{cc}\frac{\mu}{\gamma^{1/\alpha}} & ~\alpha\neq1 \\\frac{\mu}{\gamma}+\frac{2\beta\ln\gamma}{\pi} & ~\alpha=1\end{array}\right.\nonumber
\end{eqnarray}
we get, the {\it standard } form density $\tilde{Y}\sim S_{\alpha}( \beta,1,0)$ using Property~\ref{Scale}. Furthermore, samples from the transformed data set $\tilde{Y}=\{ \tilde{Y}_1,\ldots, \tilde{Y}_N\}$ are also distributed according to the following standard density :
\begin{equation}\label{scaleNshift}
f(Y|\alpha,\beta,\gamma,\mu)=f(\tilde{Y}|\alpha,\beta,1,0)\frac{d\tilde{Y}}{dY}=f(\tilde{Y}|\alpha,\beta,1,0)\frac{1}{\gamma^{1/\alpha}}\nonumber
\end{equation}
This implies that if we know the parameters $\alpha, \beta, \gamma$ and $\mu$ for the density generating $D_Y$
\begin{eqnarray}
\log\big[f(Y|\alpha,\beta,\gamma,\mu)\big]&=&\sum_{\lambda=1}^{N}\log f(Y_{\lambda}|\alpha,\beta,\gamma,\mu)\nonumber\\
&=& \sum_{j=1}^{N} \Big\{ -\frac{\log\gamma}{\alpha} + \log f(\tilde{Y_j}|\alpha,\beta,1,0)\Big\} \nonumber\\
&=& -N\Big( \frac{\log\gamma}{\alpha} + h(Y|\alpha,\beta)\Big)\nonumber
\end{eqnarray}
where, $h(Y|\alpha,\beta)$ is defined by
\begin{equation}\label{eq:hDef}
 h(Y|\alpha,\beta) \equiv -\frac{1}{N}\sum_{j=1}^{N} \log f(\tilde{Y_j}|\alpha,\beta,1,0)
\end{equation}
Here $\tilde{Y}_j$ and $Y_j$ are related via Equation~\ref{trans} for all $1\leq j\leq N$. Note that since the transformed variables $\tilde{Y}_j$ are samples from $f(\tilde{Y}|\alpha,\beta,1,0)$, we have the following asymptotic result for large $N$
\begin{eqnarray}
\lim_{N\rightarrow\infty}  h(Y,\alpha,\beta) &=& -\lim_{N\rightarrow\infty}\frac{1}{N}\sum_{j=1}^{N} \log f(\tilde{Y_j}|\alpha,\beta,1,0)\nonumber\\
&=&-\int_{-\infty}^{\infty} f(\tilde{Y}|\alpha,\beta,1,0)\log f(\tilde{Y}|\alpha,\beta,1,0) dY\nonumber\\
&=&H\Big[S_{\alpha}(\beta,1,0)\Big]\nonumber
\end{eqnarray}
where, $H[.]$ is the entropy of the corresponding random variable. 
\end{proof}
As things stand, the entropy $H[.]$ of stable random variables in the standard form is just as difficult to compute as the original log-likelihood and the previous lemma has just transformed one intractable quantity into another. However, there is an important class of models where we can ignore the entropy term during structure learning. These multivariate distributions have a special property that every linear combination of random variables is distributed as a stable distribution $S_{\alpha}(\beta,.,.)$ with the same $\alpha$ and $\beta$. One scenario when this is true is when the noise term is symmetric {\it i.e.} $\beta_i=0~\forall~X_i\in\mathcal{X}$. This special case is important since we later show (Lemma~\ref{symmetrization}) that every $\alpha$-SG model can be easily transformed into a partner symmetric $\alpha$-SG model with identical network topology and regression coefficients. For all practical purposes, learning the structure of symmetric $\alpha$-SG models is effectively the same as learning structure of arbitrary $\alpha$-SG models.

\begin{lem}\label{LCLem}Given a symmetric $\alpha$-stable graphical model for variables in $\mathcal{X}$, 
\begin{eqnarray}
Z\equiv w^T\mathcal{X}=\sum_{X_j\in\mathcal{X}} w_jX_j&\sim& S\Big( \alpha, \beta(w)=0, \gamma(w), \mu(w)\Big)~,~\forall w\in\mathbb{R}^{|\mathcal{X}|}\nonumber\\
\mathrm{if, }~\beta_i&=&0, ~\forall X_i\in\mathcal{X}\nonumber
\end{eqnarray}
\end{lem}
\begin{proof}
The dispersion $\gamma(w)$ and skewness $\beta(w)$ for the projection $w^T\mathcal{X}$ of any $d$-dimensional stable random density is given by~\citep{Samorodnitsky94}
\begin{eqnarray}
\gamma(w) &=&\int_{S_d} |w^Ts|^\alpha\Lambda(ds)\nonumber\\
\beta(w) &=& \gamma(w)^{-1} \int_{S_d}\mathrm{sign}(w^Ts)|w^Ts|^\alpha\Lambda(ds)\nonumber
\end{eqnarray}
Since, $\mathcal{X}$ represents a symmetric $\alpha$-stable graphical model, Lemma~\ref{asg} implies
\begin{eqnarray}
\beta(w)&=& \sum_{k=1}^{d}\frac{|w^Tc_k|_2^\alpha \gamma_k}{2\gamma(w)}
\int_{S_d}\Big\{\delta(s-\frac{c_k}{|c_k|_2}) +\delta(s+\frac{c_k}{|c_k|_2})\Big\}|w^Ts|^\alpha\mathrm{sign}(w^Ts)ds\nonumber\\
&=& 0\nonumber
\end{eqnarray} 
\end{proof}

We are now in a position to present the main contribution of this paper : an alternative criterion for model selection that is both computationally efficient and comes with robust theoretical guarantees (Lemma~\ref{MDC_BIC}).  The criterion is called {\it minimum dispersion criterion (MDC)} and is a penalized version of a technique previously used in signal processing literature for designing filters for heavy-tailed noise~\citep{Stuck}.

\begin{definition} Given a data set $D=\{D_1, \ldots, D_{N} \}$, the {\it penalized dispersion score} $S_{MDC}(B|D)$ for a Bayesian network $B(G,\Theta)$ is defined as,
\begin{equation}
S_{MDC}(B|D)= -\sum_{X_i\in \mathcal{X}} \Big\{N\frac{\log\gamma_i}{\alpha} + \frac{|Pa(X_i)|}{2}\log N \Big\}\nonumber
\end{equation}
The {\it minimum dispersion criterion (MDC)} selects the Bayesian network that maximizes this score over the space of all directed acyclic graphs $G$ and parameters $\Theta$.
\end{definition}

\begin{lem}\label{MDC_BIC} Given a data set $D=\{D_1, \ldots, D_{N}\}$ generated by a symmetric $\alpha$-stable graphical model, $B^{*}(G^*,\Theta^*)$, the minimum dispersion criterion is asymptotically equivalent to the Bayesian information criterion over the search space of all symmetric $\alpha$-stable graphical models\end{lem}
\begin{proof}
First consider the contribution to {\it BIC} score from each family (ie., each random variable and its parents) separately. Let $Z_j=X_j-\sum_{X_k\in Pa(X_j)} w_{jk}X_k$ be any arbitrary set of regression coefficients for a candidate network $B(G,\Theta)$. Note that the coefficients $W_j=\{w_{jk}|X_k\in Pa(X_j)\}$ need not be the true regression coefficients $W_{j}^{*}$ and $B$ need not be the true network $B^*$. We will use the notation $Z_{i,\lambda}$ for the instantiation of $Z_i$ in sample $D_\lambda\in D$. Since $D$ includes samples from a symmetric $\alpha$-stable  graphical model, Lemma~\ref{LCLem} implies $Z_j\sim S_{\alpha}( \beta=0,\gamma_j,\mu_j)$. Therefore, using Lemma~\ref{lk_gamma}
\begin{eqnarray}
Fam(X_j, Pa(X_j)|D)&\equiv& \sum_{\lambda=1}^{N}\log\Big[f(Z_{j,\lambda}|\alpha, \beta=0, \gamma_j,\mu_j)\Big]-\frac{|Pa(X_j)|}{2}\log N\nonumber\\
&=& -N\Big(\frac{\log\gamma_j}{\alpha} + h(\tilde{Z}_j|\alpha,\beta=0) \Big) -\frac{|Pa(X_j)|}{2}\log N\nonumber
\end{eqnarray}
where, as in Equation~\ref{eq:hDef}, $Z_j$ and $\tilde{Z}_j$ are related by the transformation in Equation~\ref{trans}.
\begin{eqnarray}
\implies\frac{S_{BIC}(B|D)}{N} &=& \sum_{X_j\in\mathcal{X}}\frac{Fam(X_j, Pa(X_j)|D)}{N}\nonumber\\
 &=& -\sum_{X_j\in\mathcal{X}}\Big(\frac{\log\gamma_j}{\alpha} + h(Z_j|\alpha,\beta=0) +\frac{|Pa(X_j)|}{2N}\log N\Big)\nonumber\\
\implies \lim_{N\rightarrow\infty} \frac{S_{BIC}(B|D)}{N} &=&\lim_{N\rightarrow\infty} \frac{S_{MDC}(B|D)}{N} - |\mathcal{X}|H[S_{\alpha}(\beta=0,1,0)\big]\nonumber 
\end{eqnarray}
Since, $|\mathcal{X}|H[S_{\alpha}(\beta=0,1,0)]$ is independent of the candidate network structure and regression parameters $\{W_{j}|X_j\in\mathcal{X}\}$, we get the result that for any pair of networks $B$ and $B'$
\begin{eqnarray}
\implies\lim_{N\rightarrow\infty}\frac{1}{N}\Big(S_{BIC}(B|D)-S_{BIC}(B'|D)\Big)&=&\lim_{N\rightarrow\infty}\frac{1}{N}\Big(S_{MDC}(B|D)-S_{MDC}(B'|D)\Big)\nonumber
\end{eqnarray}
Therefore, asymptotically, $BIC$ is equivalent to $MDC$ when data is generated by a symmetric $\alpha$-SG graphical model.
\end{proof}
We now show how samples from any stable graphical model can be combined to yield samples from a partner symmetric stable graphical model with identical parameters and network topology. This transformation was earlier used by \citet{Kuruoglu01} in order to estimate parameters from skewed univariate stable densities. We should point out that the procedure described above has the drawback that symmetrized data set has half the sample size.

\begin{lem}\label{symmetrization} Every $\alpha$-SG model can be associated with a symmetric $\alpha$-SG model with identical skeleton and regression parameters.
\end{lem}
\begin{proof}
Given a data set $D=\{D_1,\ldots D_N\}$ representing any $\alpha$-SG model $B(G,\Theta)$, consider a resampled data set $\widehat{D}=\{\widehat{D_1},\ldots \widehat{D_{N_S}}\}$ with variable instantiations
\begin{equation}
\widehat{X_{i,\lambda}}=X_{i,2\lambda} - X_{i,2\lambda-1}~,~\forall \lambda\in\{1,\ldots N_S=\lfloor N/2\rfloor\}\nonumber
\end{equation}
These 'bootstrapped' data samples $\widehat{D_{\lambda}}=\{\widehat{X_{i,\lambda}}|X_i\in\mathcal{X}\}$ represent independent instantiations of random variables $\widehat{\mathcal{X}}\equiv\{\widehat{X_i}|X_i\in\mathcal{X}\}$.  Similarly, we may use the regression parameters $W$ to define resampled noise variables :
\begin{equation}\widehat{Z_j}\equiv\widehat{X_j}-\sum_{\widehat{X_k}\in Pa(X_j)} w_{jk}\widehat{X_k}\nonumber
\end{equation}
We now make two observations :
\begin{enumerate}
\item If $Z_j=X_j-\sum_{X_k\in Pa(X_j)} w_{jk}X_k\sim S_{\alpha}( \beta_j,\gamma_j,\mu_j)$, then using Property~\ref{Aggregation}
\begin{equation}\widehat{Z_j}\equiv\widehat{X_j}-\sum_{\widehat{X_k}\in Pa(X_j)} w_{jk}\widehat{X_k}\sim S_{\alpha}( \beta=0,2\gamma_j,0)\nonumber
\end{equation}
\item The transformed noise variables $\widehat{Z_j}$ are independent of each other.
\end{enumerate}
But these conditions define an $\alpha$-SG model (Definition~\ref{ASGdef}). Therefore, by Lemma~\ref{BN}, the resampled data is distributed according to a Bayesian network $\widehat{B}(G,\widehat{\Theta})$ such that
\begin{eqnarray}
\widehat{Z_j} &\equiv& \widehat{X_j}-\sum_{\widehat{X_k}\in Pa(X_j)} w_{jk}\widehat{X_k}\nonumber\\
P_{\widehat{B}}(\widehat{\mathcal{X}})&=&\prod_{j=1}^{|\mathcal{X}|}f(\widehat{Z_j}|\alpha,0,2\gamma_j,0)\nonumber\\
\widehat{\theta_j}&=&\{\alpha,\beta=0,2\gamma_j,0\}\cup W_j, ~\widehat{\Theta}=\{\widehat{\theta_j}|X_j\in\mathcal{X}\}\nonumber
\end{eqnarray}
\end{proof}
 
\subsection{The \texttt{StabLe} Algorithm }
In this section we describe \texttt{StabLe}, an algorithm for learning the structure and parameters of $\alpha$-SG models (Algorithm~\ref{alg:StabLe}). The first step of  \texttt{StabLe} is to center and symmetrize the entire data matrix $D_I$ in terms of the variables $\widehat{\mathcal{X}}$, as described in Lemma~\ref{symmetrization}. This is followed by estimating the global parameter $\alpha$ using the method of $\log$ statistics~\citep{Kuruoglu01}.  Finally, structure learning is performed by  a modified version of the {\it ordering-based search} (OBS) algorithm (Section~\ref{OBS}). The details of parameter estimation and structure learning algorithms are discussed next.
\begin{algorithm}[htb!]
   \caption{{\tt StabLe }}
   \label{alg:StabLe}
\begin{algorithmic}
   \STATE {\bfseries Input:} Input data matrix $D_I$, number of random restarts Nreps
   \STATE {\bfseries Output:} $\alpha$-SG model $B(G, \Theta)$ over $\mathcal{X}$	
   \STATE $D \leftarrow Symmetrized(D_I)$ \hspace{3.2cm}{\tt // Symmetrize the data as per Lemma~\ref{symmetrization}}
  \STATE Estimate $\alpha$ from $D$ \hspace{4.9cm}{\tt // Use log-statistics,  Equation~\ref{eq : alpha_est}} 
  \STATE Initialize $B(G,\Theta) = \emptyset$
  \FOR{i =1 {\bfseries to} Nreps} 
  \STATE Initialize a random ordering $\sigma$
\STATE $B_\sigma(G,\Theta)=OBS(D,\alpha,\sigma)$ \hspace{2.5cm}{\tt // Ordering-based search,  Algorithm~\ref{alg:OBS}} 
\IF{$S_{MDC}(B_\sigma|D)>S_{MDC}(B|D)$}
\STATE $B=B_\sigma$
\ENDIF
\ENDFOR
\end{algorithmic}
\end{algorithm}

\subsubsection{Parameter Learning}\label{ParamL}
 
First, we describe the algorithms {\tt StabLe} uses to estimate the characteristic exponent $\alpha$ from the data matrix $D$, as well as the parameters $\Gamma=\{\gamma_j|X_j\in\mathcal{X}\}$ and $W_j=\{w_{jk}|X_k\in Pa(X_j)\}$ for  any given  directed acyclic graph $G$. 

\paragraph{ Estimating the global parameter $\alpha$ :}
Log statistics can be used to estimate the characteristic exponent $\alpha$ from the centered and symmetrized variables in $\widehat{\mathcal{X}}$~\citep{Kuruoglu01}. 
\subparagraph{Algorithm:}Since every linear combination of variables in $\widehat{\mathcal{X}}$ has the same $\alpha$, if we define 
\begin{displaymath}\widehat{X}=\sum_{i=1}^{|\widehat{\mathcal{X}}|}\widehat{X_i}~,\mathrm{~then}\end{displaymath}
\begin{eqnarray}\label{eq : alpha_est}
\alpha&=&\Big(\frac{L_2}{\psi_1}-\frac{1}{2}\Big)^{-1/2}\\
L_2 &\equiv&\mathbb{E}\Big[\big(\log|\widehat{X}|-\mathbb{E}[\log|\widehat{X}|]\big)^2\Big]\nonumber\\
\psi_1&\equiv&\frac{d^2}{dy^2}\Gamma(y)\bigg|_{y=1}=\frac{\pi^2}{6}\nonumber
\end{eqnarray}

\paragraph{Estimating the dispersion $\gamma_j$,  and regression parameters $W_j =\{w_{jk}| X_k\in Pa(X_j)\}$ }
 If $ \gamma_j(W_{j})$ is the dispersion parameter for the distribution of $Z_j = X_j-\sum_{X_k \in Pa(X_j)} w_{jk} X_{k} $, then the minimum dispersion criterion selects regression parameters
\begin{equation}
W^{*}_{j}=\arg\min \frac{1}{\alpha} \log \gamma_j(W_{j})\nonumber
\end{equation}
Minimum dispersion regression coefficients are estimated using a connection between the $l_p$-norm of a stable random variable and the dispersion parameter $\gamma$~\citep{Zolotarev57, Kuruoglu01}. 
\begin{lem}\label{FLOM}If $Z\sim S_{\alpha}(0,\gamma,0)$, then 
\begin{equation}
E(|Z|^p)= C(p,\alpha)\gamma^{p/\alpha} ~~\forall-1<p<\alpha\nonumber
\end{equation}
where,
\begin{equation}
C(p,\alpha)=\frac{\Gamma(1-\frac{p}{\alpha})}{\Gamma(1-p)\cos(p\frac{\pi}{2})}\nonumber
\end{equation}
\end{lem}
Therefore, to within a constant term $\log C(p,\alpha)$, minimizing  $\frac{1}{\alpha}\log \gamma_j$ is identical to minimizing the   $l_p$-norm $\|Z_j\|_p\equiv(\sum_{\lambda=1}^{N}|Z_{j,\lambda}|^p)^{1/p} $  for $-1<p<\alpha$. 
\begin{equation}
W^{*}_{j}=\arg\min \log \Bigl(\|Z_j\|_p\Bigr)\equiv \arg\min \log \Bigl((\sum_{\lambda=1}^{N}|Z_{j,\lambda}|^p)^{1/p}\Bigr) \nonumber
\end{equation}
\begin{algorithm}[tb]
   \caption{IRLS {\tt //~Find the least $l_p$ norm regression coefficients}}
   \label{alg:IRLS}
\begin{algorithmic}
   \STATE {\bfseries Input:} $N$ dimensional vector for instantiations of the child node $Y$,  $N\times M$ matrix  $X$ of instantiations of the parent set $Pa(Y)$, tolerance $\epsilon$ and  $p\in(0,2]$
   \STATE {\bfseries Output:}  $M$ dimensional vector  of regression co-efficients $ W^{*} = \arg\min_{W} \| Y- XW\|_{p}$
   \STATE Initialize $W$ with OLS co-efficients $W= (X^{T}X)^{-1}(X^{T}Y)$
         \REPEAT
   \STATE Initialize buffer for current regression coefficients $\beta= W$
   \STATE Initialize a diagonal $N\times N$ matrix $\Omega$ from $\beta$ for weighted least squares regression
   \begin{displaymath}
  \Omega_{ij}=\delta_{ij}(Y_{i}-(XW)_{i})^{p-2}~\forall i,j\in\{1,\ldots N\}
   \end{displaymath}
   \STATE Update regression coefficients vector $W=(X^{T}\Omega X)^{-1}(X^{T}\Omega Y)$
   \UNTIL{$\|\beta-W\|_{2}<\epsilon$}~ {\tt // Change in regression coefficients is within tolerance}
\end{algorithmic}
\end{algorithm}
\subparagraph{Algorithm:}Minimization of the $l_p$ norm is performed by the {\it iteratively least squares} (IRLS) algorithm~\citep{byrd1979convergence}. Briefly, the IRLS algorithm repeatedly solves an instance of the weighted least squares problem to achieve successive estimates for the least $l_p$ norm coefficients (Algorithm~\ref{alg:IRLS}).  IRLS is attractive since rigorous convergence guarantees can be given~\citep{Daubechies} and the method is easy to implement since  several software packages are available for the weighted least squares problem. Even though the IRLS objective is no longer convex for $p<1.0$, \citet{Daubechies} show that under certain sparsity conditions, the algorithm can recover the true solution. Simulations described in Section~\ref{Simulations} tend to support this observation.

For experiments described in this manuscript, {\tt StabLe} used two values of $p$ for $l_p$-norm estimation. For learning regression coefficients during structure learning, IRLS was implemented with $p=\alpha/1.01$, since lower values tended to give noisier estimates (possibly due to numerical errors). However, we also found that estimating the term $\log C(p,\alpha)$ is prone to numerical errors for small values of $|\alpha-p|$. Therefore, we ignore this constant term during structure learning since it is common to all candidate structures. {\tt StabLe} estimates the dispersion parameters $\gamma_j$ after structure learning,  by computing the $l_p$-norm for $p=\alpha/10.0$ and applying Lemma~\ref{FLOM}.



\begin{algorithm}[tb!]
   \caption{K2Search}
   \label{alg:K2}
\begin{algorithmic}
   \STATE {\bfseries Input:} Symmetrized data matrix $D$, fixed ordering $\sigma$ and shape parameter $\alpha$
     \STATE {\bfseries Output:} $\alpha$- SG model $B(G, \Theta)$ given the ordering $\sigma$
\STATE Initialize $B(G, \Theta) = \emptyset$
   \FOR{$i=2$ {\bfseries to} $|\mathcal{X}|$}
   \STATE {\tt //~Find the optimal parent set $Pa(\sigma_i)$ by greedily }\\
  {\tt //~ adding edges starting from $Pa(\sigma_i)=\emptyset$}
  \REPEAT
   \STATE Initialize $noChange = true$    
   \STATE Initialize $best= FS(\sigma_i,Pa(\sigma_i)|D, \alpha)$
   \STATE $AddPa=\emptyset$ \hspace{5.6cm}{\tt //~Search for a potential parent}
   \FOR{$X_j\in \{\sigma_1\ldots \sigma_{i-1}\}\setminus Pa(\sigma_i)$}   
   \STATE Estimate regression weights $W_{\sigma_i}$ for parent set $Pa(\sigma_i)\cup X_j$ using IRLS							
  	 \IF{$FS(\sigma_i,Pa(\sigma_i)\cup X_j|D, \alpha)> best$}
   \STATE $best=FS(\sigma_i, Pa(\sigma_i)\cup X_j|D, \alpha )$	 \hspace{3cm} {\tt //~Update best score and } 		
   \STATE $AddPa=X_j$	\hspace{6.2cm}{\tt // possible new parent }
   \STATE $noChange = false$
   \ENDIF
   \ENDFOR
   \STATE $Pa(\sigma_i)= Pa(\sigma_i)\cup AddPa$  \hspace{5.2cm} {\tt //~Add the new parent} 	
   \UNTIL{$noChange$ is $true$}   \hspace{4.4cm}{\tt //~Repeat until local optimum}
   \ENDFOR

\end{algorithmic}
\end{algorithm}

\begin{algorithm}[tb!]
   \caption{OBS {\tt //~Find the optimal $\alpha$-SG model using OBS}}
   \label{alg:OBS}
\begin{algorithmic}
   \STATE {\bfseries Input:} Symmetrized data matrix $D$, shape parameter $\alpha$, initial ordering $\sigma$
   \STATE {\bfseries Output:} $\alpha$-SG model $B(G, \Theta)$ over $\mathcal{X}$	
   	
\STATE Initialize SG model $B$=K2Search($D$, $\sigma$, $\alpha$)\hspace{2cm}
	\FOR{$i=1$ {\bfseries to} $|\mathcal{X}|-1$}
    	  \STATE Initialize $T_i\sigma =Twiddle(i,\sigma)$ \hspace{1.2cm}{\tt //~New ordering $T_i\sigma$ by swapping  ${\sigma}_i$ \& $\sigma_{i+1}$}
  	  \STATE $\tilde{B}$= K2Search($D$, $T_i\sigma$, $\alpha$) \hspace{2.9cm}{\tt //~Compute the optimum $\tilde{B}$ given $T_i\sigma$}
\STATE $DS(i)=S_{MDC}(\tilde{B}|D) - S_{MDC}(B | D)$  \hspace{1cm}{\tt //~ Set Delta score for the twiddle}
 	\ENDFOR

  \REPEAT 
  	\STATE Initialize $noChange = true$ 
	\STATE Find $a = \arg\max DS(i)$    \hspace{4.7cm}{\tt //~Find the best twiddle $T_a{\sigma}$}
  	  \STATE $\tilde{B}$= K2Search($D$, $T_a{\sigma}$, $\alpha$) \hspace{3.5cm}{\tt //~Compute the optimum given $T_a{\sigma}$}
  		\IF{$S_{MDC}(\tilde{B}|D)> S_{MDC}(B | D)$} 
		   \STATE $\sigma=T_a\sigma, ~ B=\tilde{B} $\hspace{4.4cm}	{\tt  //~Accept the swap and update $\sigma, B$}
		   	\STATE	 $DS(a-1)$ (if $a>1$) \hspace{2.cm}	{\tt  //~Update  delta scores for neighbors $a-1$}
			\STATE $DS(a+1)$  (if $a<|\mathcal{X}|-1$) \hspace{4.9cm}{\tt  //~ and $a+1$, if valid}
		   \STATE $noChange = false$
   		\ENDIF
  \UNTIL{$noChange$ is $true$}\hspace{4.7cm}{\tt //~Repeat until local optimum}
  
\end{algorithmic}
\end{algorithm}

\subsubsection{Structure Learning}\label{OBS} 
Searching the space of all network structures can be performed through any of the popular hill-climbing algorithms. In this paper we used the {\it ordering-based search} (OBS) algorithm~\citep{Teyssier} to search for a local optimum in the space of all directed acyclic graphs.  The algorithm starts with an initial ordering $\sigma$ and then learns a DAG consistent with $\sigma$ ( i.e., all parents of each node must have a lower order). This part of structure learning is performed via a subroutine K2Search (Algorithm~\ref{alg:K2}), which is a modified version of the hill-climbing based K2Search algorithm~\citet{Cooper}. K2Search starts with an empty parent set for each node $X_i\in\mathcal{X}$ and greedily adds edges until the MDC based score $FS(X_i, Pa(X_i)|D, \alpha) = -\frac{N}{\alpha} \log \gamma_i -\frac{|Pa(X_i)|}{2}\log N $ reaches a local maximum. The main difference from Gaussian graphical models~\citep{Heckerman00, Schmidt} is that K2Search scores each family based on least $l_p$ norm instead of ordinary least squares (OLS). Once K2Search has learned the locally optimum DAG for a given ordering $\sigma$, OBS explores other ordering by performing elementary operations (or `twiddles') that swap the order of successive variables and recomputes the K2Search scores. This process is continued until a local optimum. {\tt StabLe} also performs a fixed number of random restarts to explore more of the search space. In all experiments reported here we used 10 random restarts. Pseudo code for the methods is described in Algorithms~\ref{alg:OBS} and \ref{alg:K2}.

\section{Empirical Validation}\label{Validation}
In this section we describe two sets of numerical experiments to assess the performance of {\tt StabLe}. The first set is based on synthetic data representing five benchmark network topologies (Section~\ref{Simulations}). These experiments test the accuracy and robustness of MDC based learning on simulated data sets where the ground truth (structure and parameters) is known.

For the second set of experiments, we apply {\tt StabLe} to a gene expression data set  (Section~\ref{GeneEx}) from Phase III of the HapMap project~\citep{HapMap3}. These samples represent microarray measurements of mRNA expression within lymphoblastoid cells from 727 individuals belonging to eight global population groups~\citep{montgomery2010transcriptome, stranger2012patterns}.

For structure learning, we chose ordinary least squares (OLS) based BIC penalized log-likelihood $S_{OLS}(B|D)$ for comparison. 
\begin{equation}\label{eq: S_OLS}
S_{OLS}(B|D)= -\sum_{X_i\in\mathcal{X}}\Bigl \{ \log \|Z_i - \bar{Z}_i\|_2 + \frac{|Pa(X_i)|}{2}\log N \Bigr\}
\end{equation}
OLS is commonly used for learning Gaussian graphical models and should be identical to {\tt StabLe} for $\alpha=2.0$ (for that case $S_{OLS}$ and $S_{MDC}$ are the same up to a network and parameter independent term). This comparison allowed us to asses the effect of heavy-tailed noise ($\alpha<2.0$) on learning performance.

\begin{figure}[htb!]
\begin{center}
\includegraphics[scale=.56, angle=-90]{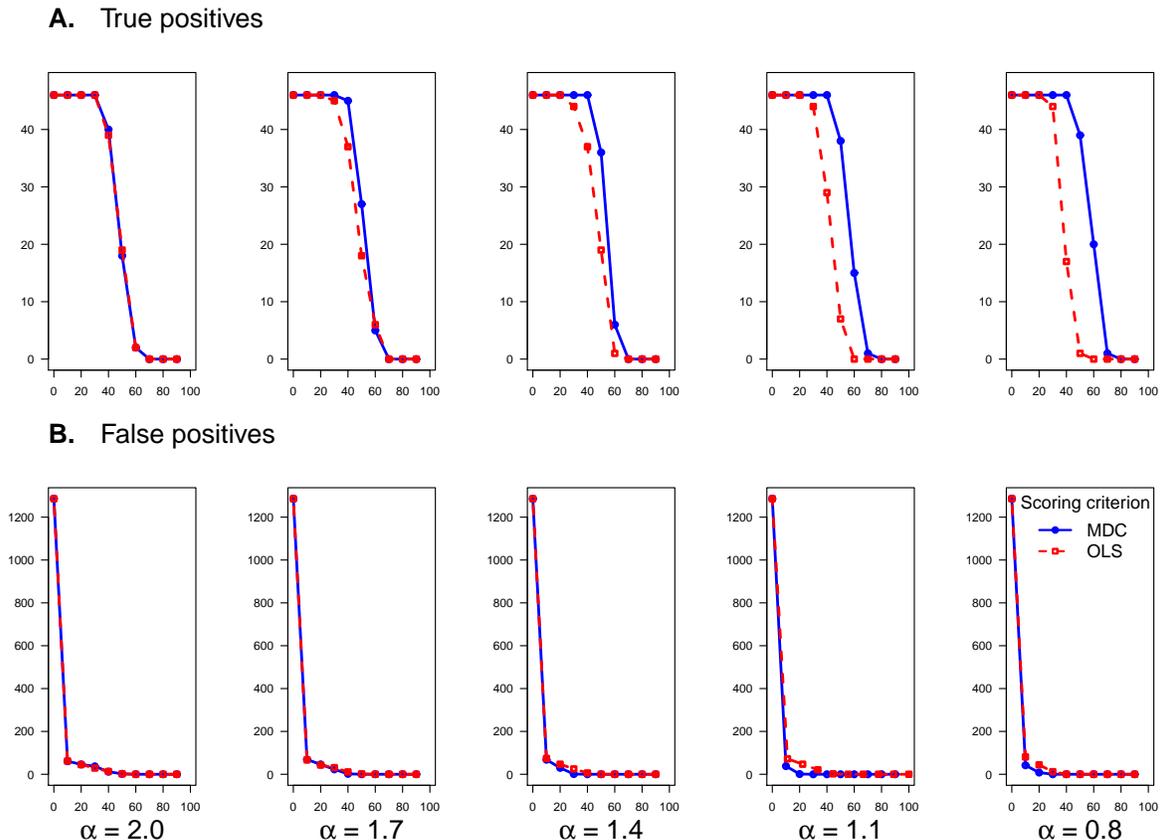}
\end{center}
\caption{The \texttt{ALARM} network - Inferred structure. Comparative performance of MDC based {\tt StabLe} algorithm (solid blue curves) versus an identical algorithm based on OLS score (dashed red curves). Vertical axes show true positives in {\bf A} and false positives in {\bf B}, for directed edges present in the input network. Horizontal axes show respective confidence (percentage of simulated data sets with the feature)}\label{fig: ALARM_TPFP}
\end{figure}

\begin{figure}[htb!]
\begin{center}
\includegraphics[scale=.4, angle=-90]{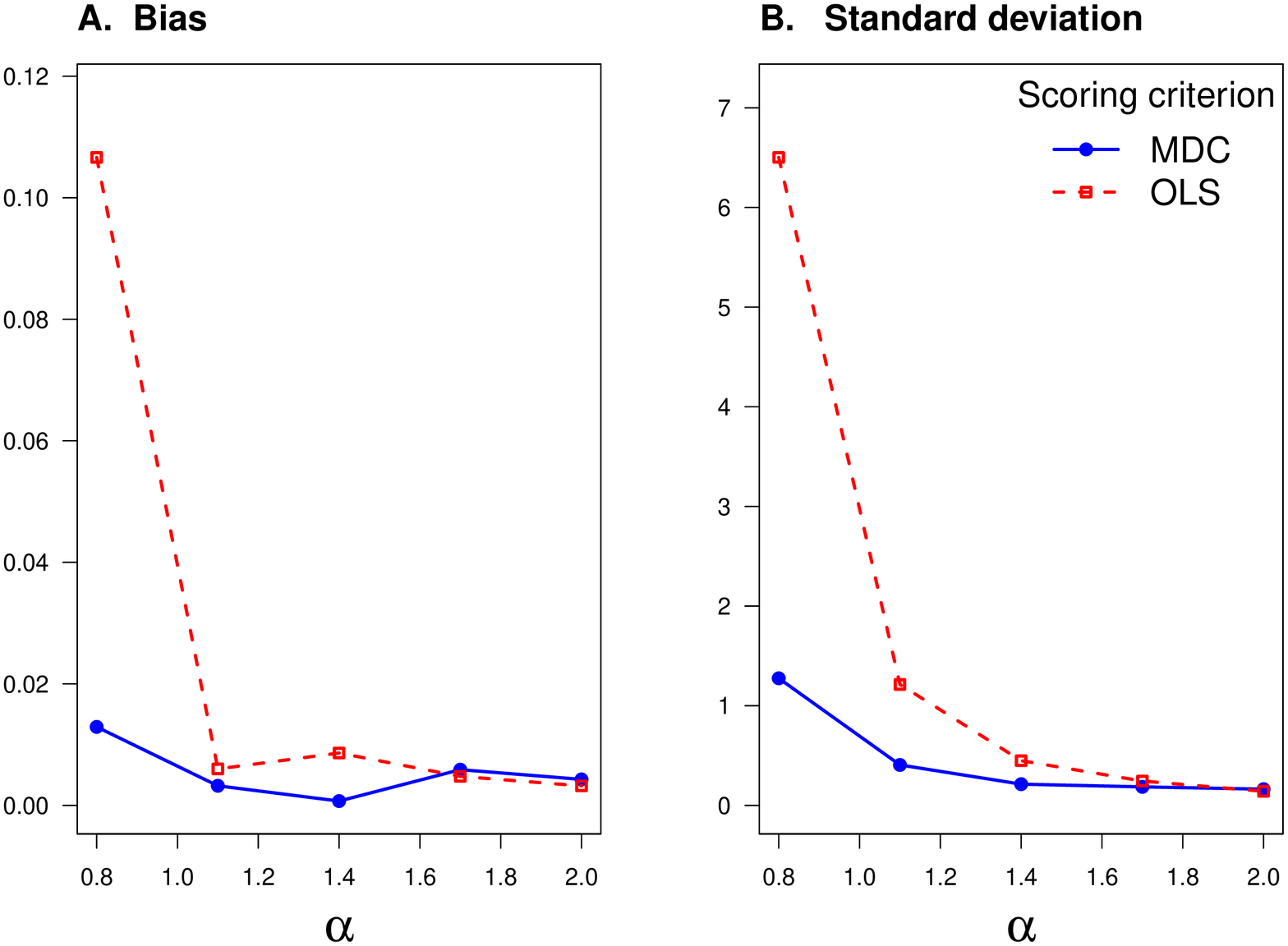}
\end{center}
\caption{The \texttt{ALARM} network - Estimated regression parameters.}\label{fig: ALARM_Reg}
\end{figure}

\begin{figure}[htb!]
\begin{center}
\includegraphics[scale=.56, angle=-90]{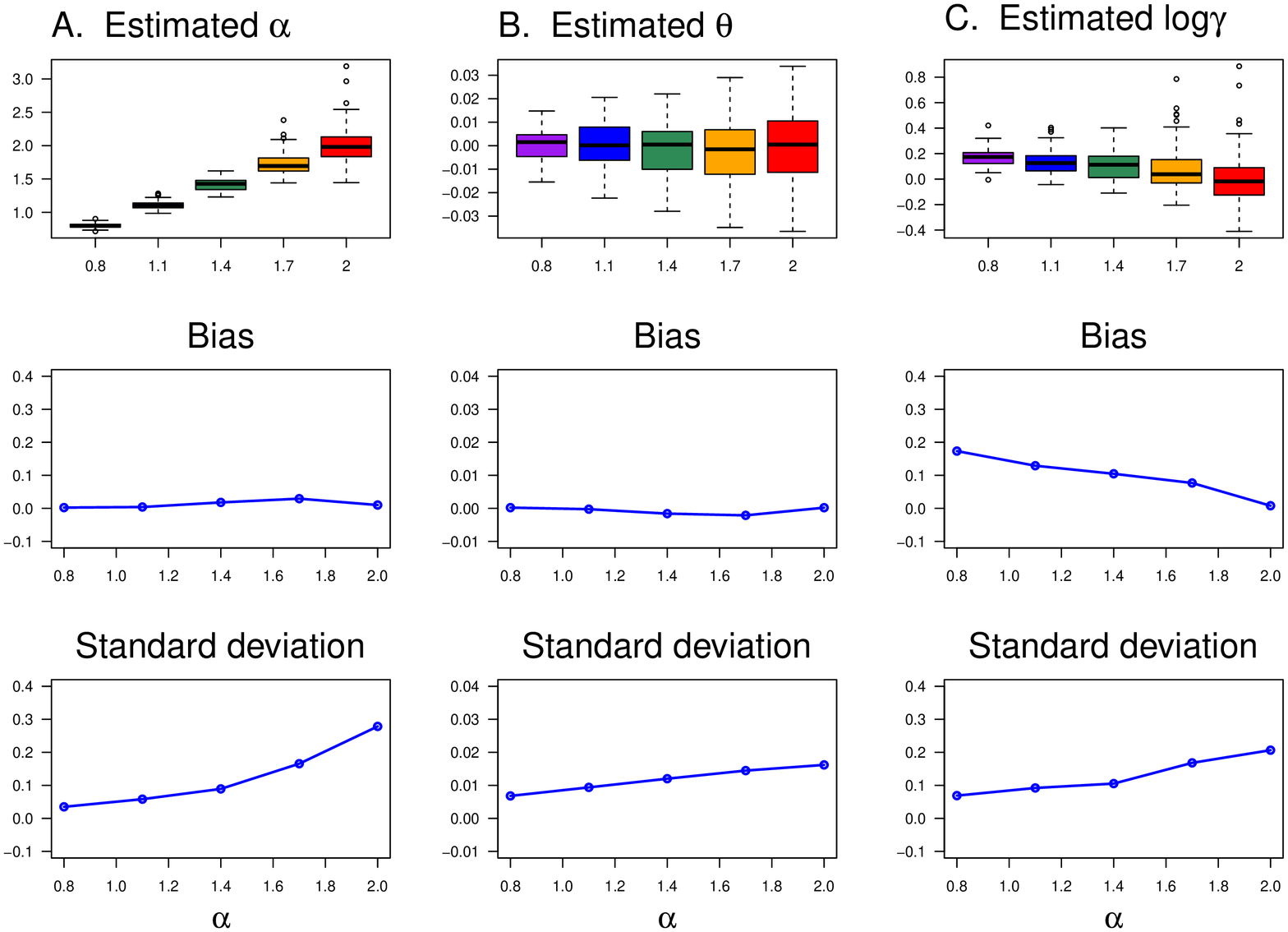}
\end{center}
\caption{The \texttt{ALARM} network - Estimated noise parameters.}\label{fig: ALARM_ABG}
\end{figure}

\subsection{Synthetic Data}\label{Simulations}
We performed numerical experiments based on simulated data sets for five network topologies from the Bayesian network repository~\footnote{A description for each network is available at {\tt http://www.cs.huji.ac.il/labs/compbio/Repository}.}. These were (number of nodes, edges within brackets) : {\tt ALARM} (37, 46), {\tt BARLEY} (48, 84), {\tt CHILD} (20, 25),  {\tt INSURANCE} (27, 52) and {\tt MILDEW} (35, 46). Adjacency matrix for each network was downloaded from the supplement to \cite{tsamardinos2006max}\footnote{Supplement can be accessed at {\tt http://www.dsl-lab.org/supplements/mmhc\_paper/mmhc\_index.html}.}. Each node $X_i\in\mathcal{X}$ was assigned an additive $\alpha$-stable noise variable $Z_i$ with same parameters $S_{\alpha}(\beta, \gamma,0)$ and each edge was assigned a regression coefficient that was sampled from $[-\frac{\rho}{2}, +\frac{\rho}{2}]$  uniformly at random. The $S_{\alpha}(\beta, \gamma,0)$ noise variable was simulated using the method of \citet{Chambers76}. For each set of experiments, we simulated 100 datasets, each with 2000 samples from an $\alpha$-SG model with randomly chosen regression weights, but fixed network topology and $\alpha$-stable noise parameters. The goal was to asses {\tt StabLe} in terms of its performance at structure learning and estimation of stable noise parameters for a variety of regression coefficients.

We performed five sets of experiments for each network, corresponding to different values of $\alpha$ = 0.8, 1.1, 1.4, 1.7, 2.0. For each set of experiments, we chose $\rho=1.0,~ \beta =0.9$ and $\gamma =1.0$. We chose such a high skew ($\beta=0.9$) in the input data to test our algorithm on its ability to symmetrize and correctly learn (possibly) difficult problem instances. Instead of $\beta$ however, we report a related parameter $\theta = \arctan( \beta \tan\alpha\frac{\pi}{2})$ which can be inferred more robustly in practice since it avoids the singularity near $\alpha =2$~\citep{Kuruoglu01}. We used the zeroth order signed moments based method for estimating $\theta$~\citep{Kuruoglu01}.

\begin{equation}\label{eq: thetaEst}
\theta_i =\frac{\alpha\pi}{2N}\sum_{\lambda=1}^{N} \mathrm{sign}(X_{i,\lambda}),~\forall ~X_i\in\mathcal{X}
\end{equation}


We report two set of results for each network : structure learning and parameter estimation. For convenience, we describe the results for the {\tt ALARM} network first (results for other data sets are provided in Appendix~B). 

\subsubsection{Inferred Structure}
Figure~\ref{fig: ALARM_TPFP} shows the comparative performance of MDC and OLS based approaches. Each curve shows the number of inferred directed edges. Figure~\ref{fig: ALARM_TPFP}A, B  show the number of true positives and true negatives at a given confidence level (percentage of simulated data sets where the directed edge was learnt). Solid (blue) curves show the performance of MDC and dashed (red) curves show OLS based method. The results are along expected lines with the difference between the two getting larger as $\alpha$ is varied away from 2.0. One clear trend is that while the sensitivity to true positive detection degrades for OLS (Type II errors) as $\alpha$ decreases, the MDC based method remain robust to changes in $\alpha$. Both methods are however quite reliable at not inferring incorrect edges (false positives or Type I errors). Similar behavior is observed for other data sets as well (Appendix B).

\subsubsection{Estimated Parameters}
Figure~\ref{fig: ALARM_Reg} shows the comparative performance of MDC and OLS scores in estimating regression coefficients. Figure 2A shows the bias in mean estimates (in absolute magnitude) and Figure 2B, the standard deviation around the mean in estimated coefficients and are averaged over all true positives and all simulated data sets. Note that each of the 100 simulated data set had regression coefficients sampled independently from $[-1/2, 1/2]$. Both methods perform well in terms of low bias, but OLS had a much higher standard deviation at low $\alpha$. As with structure learning, this pattern was consistently observed for other network topologies as well~(Appendix B).

We also assessed the ability of {\tt StabLe} to infer $\alpha$-stable noise parameters accurately and robustly. However, we could not show a comparative performance since OLS scores assume Gaussian noise. Figure~\ref{fig: ALARM_ABG} shows the box plot and basic statistics for the estimates for $\alpha$, $\theta$ and $\log\gamma$ from the symmetrized data set (node specific parameters $\theta$ and $\log\gamma$ are reported as averages).
\begin{eqnarray*}
\alpha~,& \theta \equiv \frac{1}{|\mathcal{X}|}\sum_i\arctan( \beta_i \tan\alpha\frac{\pi}{2}) ~,& \log \gamma =  \frac{1}{|\mathcal{X}|}\sum_i \log \gamma_i
\end{eqnarray*}

Both $\alpha$ and $\theta$ estimates have low bias and standard deviation for each of the five data sets.  But, $\log \gamma$ estimates show a clear tendency to overestimate the dispersion in noise at very low $\alpha$. This is however a difficult parameter domain for most existing methods for parameter estimation, even for univariate $\alpha$ stable densities~\citep{Kuruoglu01}. As with other inferences, the performance of {\tt StabLe} is again robust to changes in network topology~(Appendix B).

\subsection{Gene Expression Microarray Data}\label{GeneEx}
In this section, we describe two sets of analyses for gene expression microarray data from phase III of the HapMap project\footnote{Data sets can be downloaded from the Array Express database {\tt http://www.ebi.ac.uk/arrayexpress/} using Series Accession Numbers E-MTAB-198 and E-MTAB-264.}.  Our approach models the set of gene expression profiles as a multivariate stable distribution that can be represented by an $\alpha$-SG model. The first set of experiments aimed at comparing the prediction accuracy of MDC with OLS-based structure learning via ten-fold cross-validation (Section~\ref{CV}). The results of these experiments establish the utility of heavy-tailed models for gene expression profiles.

Next, we apply $\alpha$-SG models to the problem of quantifying differential expression (DE) of a gene between samples belonging to different conditions. This is a common task in gene expression-based analyses in contemporary genomics. However, popular methods for detecting differentially expressed genes usually assume the expression profile for each gene to be independent of others. Based on this assumption, DE quantification is performed by testing the null hypothesis that the $\log$-expression of each gene is identical across the observed conditions and using the corresponding p-value as a measure of DE. In Section~\ref{DE}, we develop {\tt SGEX}, a new technique for quantifying differential expression of each gene that is based on $\alpha$-SG models. We apply {\tt SGEX} to quantify the DE for a gene in each population group within the HapMap data. Contrary to most existing methods, {\tt SGEX} takes into account both the heavy-tailed behavior of gene expression densities, as well as  linear dependencies between mRNA expression of different genes.

\begin{table}[tb!]
\begin{center}\begin{tabular}{ccccc}\hline ID & Ethnicity & Location & \# Samples &\# Genes/Probes\\\hline CEU & Caucasians & Utah, USA & 109 & 21800 \\CHB & Han Chinese & Beijing, China & 80 & 21800 \\GIH & Gujarati Indians & Houston, USA & 82 & 21800 \\JPT & Japanese & Tokyo, Japan & 82 & 21800 \\LWK & Luhya & Webuye, Kenya & 83 & 21800 \\MEX & Mexican & Los Angeles, USA & 45 & 21800 \\MKK & Maasai & Kinyawa, Kenya & 138 & 21800 \\YRI & Yoruba & Ibadan, Nigeria & 108 & 21800 \\\hline \end{tabular} \caption{The HapMap III population groups and selected microarray probes  as reported by~\citet{stranger2012patterns}.}
\end{center}
\label{Populations}
\end{table}

\subsubsection{Data Normalization}\label{Normalization}
We downloaded pre-processed data for 727 individuals from eight global population groups as reported in \citet{stranger2012patterns}. Details about the eight population groups are provided in Table~\ref{Populations}. For each individual, the input data represented $\log$-intensities for 21800 microarray probes\footnote{Each selected probe mapped to a unique, autosomal Ensembl gene. Ensembl gene IDs are available at {\tt http://www.ensembl.org}.} that were quantile and median normalized, as described in the original paper \citep{stranger2012patterns}. These microarray intensities provide a measure  for mRNA concentration within a sample of lymphoblastoid cells from each individual. Before performing structure learning, we further processed each probe intensity as follows :

\begin{enumerate}
\item The log-intensity  $l(i)$ for each probe $i$ was median-centered to obtain transformed log-intensities $ml(i)$, ie., the number of samples with positive log-intensity was half (or 0.5 less than) the total (=363=$\lfloor727$). This is a standard technique for learning Gaussian graphical models from gene expression data and does not affect the network structure.
\item The median-centered log-intensities were used to assign a rank $R(i)$ to each probe $i$, in decreasing order of variance. Even for $\alpha$-stable distributions, variance of $\log$ transformed data is finite~\citep{Kuruoglu01}. This is also a standard technique for restricting computing time by selecting a subset of genes with most variation.
\item The median-centered log-intensities $\{ml(i)|R(i)\leq21800\}$ were exponentiated to $ \mathcal{I}=\{2^{ml(i)}|R(i)\leq21800\}$.
\item The exponentiated-median-centered log-intensities $\mathcal{I}_k=\{2^{ml(i)}|R(i)\leq k\leq21800\}$ for the top $k$ ranked probes were provided as input to {\tt StabLe} (for cross-validation) and {\tt SGEX} (for DE quantification, as described in Section~\ref{DE}). In the experiments reported here $k=100$.
\end{enumerate}

We estimated $\alpha$ over 1000 resampled bootstrap replicates of the data. This was meant as a diagnostic to assess the heavy-tailed nature of the intensities. As shown in Figure~\ref{fig: CV}A,  the data suggests a clear departure from a Gaussian profile.
\begin{figure}[htb!]
\begin{center}
\includegraphics[scale=.4]{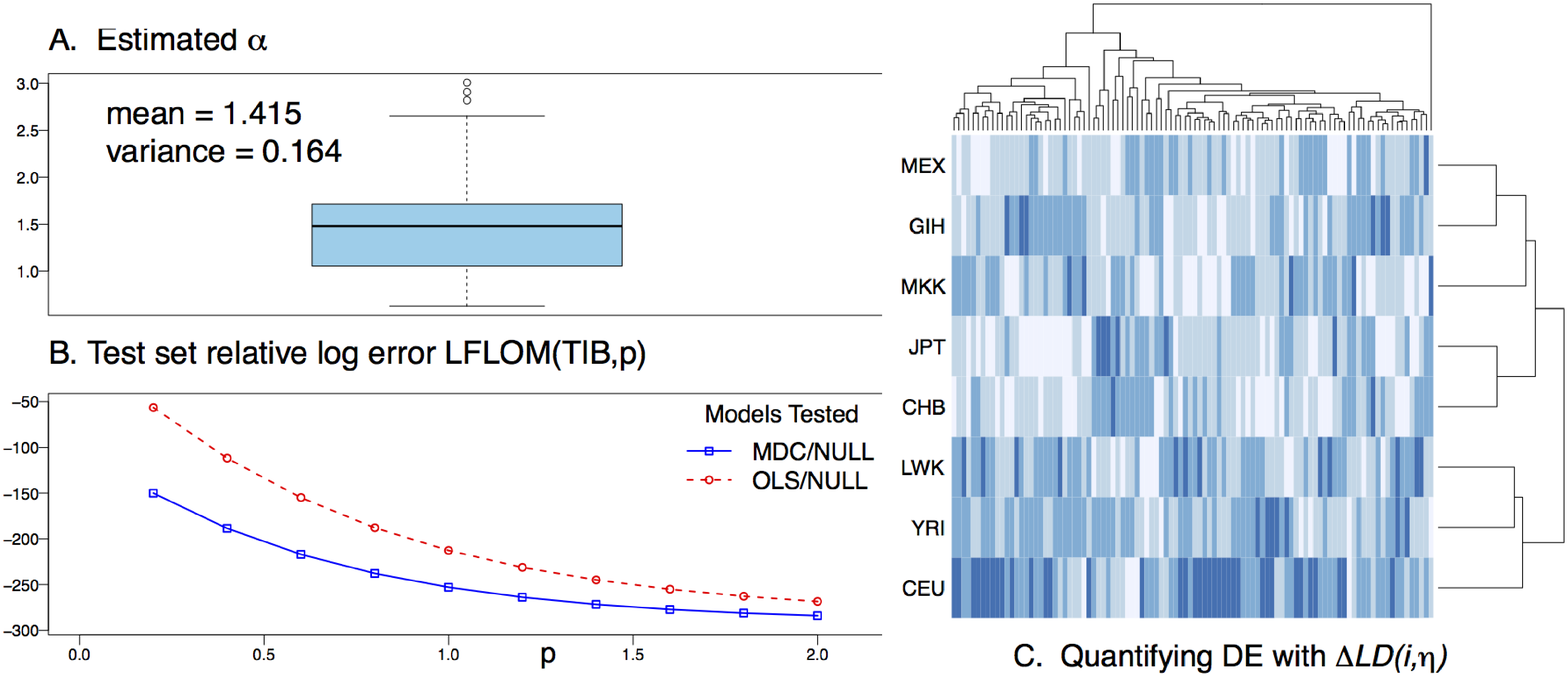}
\end{center}
\caption{Test set performance and differential expression quantification with {\tt SGEX}. {\bf A} shows a box plot of estimated $\alpha$ over 1000 bootstrap replicates. {\bf B} shows comparative Test set performance for MDC and OLS based networks relative to an empty network (no edges). {\bf C} shows a heat map of   $\Delta LD$ that quantifies differential expression of a gene. The color for each column is normalized by scaling and centering.}\label{fig: CV}
\end{figure}

\subsubsection{Cross-validation Analysis}\label{CV}
We performed a ten-fold cross-validation for the top 100 ranked probes from the HapMap data. Since we wanted to compare MDC with OLS-based learning, we report goodness of fit of the graphical model $B$ on the test set $T=\{T_1,\ldots T_N\}$ in terms of $\log$ fractional lower order moments :
\begin{equation}\label{LFLOM}
LFLOM(T|B, p)=\sum_{X_i\in\mathcal{X}}\Bigl[\frac{1}{p} \log E(|Z_{i}|^{p}) \Bigr]=\sum_{X_i\in\mathcal{X}}\Bigl[\frac{1}{p} \log E(|X_{i}-\sum_{X_j\in Pa(X_i)}w_{ij}X_j|^{p}) \Bigr]\nonumber
\end{equation}
where, $w_{ij}$ represents the regression co-efficient for the edge $(X_j, X_i)$. Clearly, if most of the variation in $X_i$ can be explained by the parent set $Pa(X_i)$, the corresponding $LFLOM$ will be small. For $p=2$, $LFLOM$ is identical to the negative log-likelihood for Gaussian graphical models\footnote{ Note that the noise term $Z_i$ has zero mean, since the data is centro-symmetrized before cross-validation.}. However, the second order moment diverges for $\alpha<2$ (Lemma~\ref{FLOM}). Therefore, $LFLOM$ provide a more robust estimate for evaluating the model on test set for heavy-tailed noise ($\alpha<2$).

Figure~\ref{fig: CV}B shows the average (over the ten-folds) of $LFLOM$ for MDC (blue) and OLS-based (red) models. In each case, the curves show the difference in $LFLOM$ between optimal (MDC or OLS) network and an empty network (NULL). This allows us to also assess the deterioration in test set performance by treating each gene as an independent random variable (a common assumption in DE quantification). Although the data set contains only 727 samples, we see a clear improvement in test set performance of $\alpha$-SG models (MDC curve) relative to Gaussian graphical models (OLS curve).

\subsubsection{Quantifying Differential Expression With {\tt SGEX}}\label{DE}
Finally, we discuss {\tt SGEX}, a new technique for quantifying differential expression using $\alpha$-SG models. {\tt SGEX} is based on cross-validation for assessing DE of a gene across different conditions. For the HapMap data, we chose each of the eight population groups in turn as the test set and learnt the optimal $\alpha$-SG model for the rest of the samples. We then estimated $\Delta LD(i,\eta)$, the change in negative log-likelihood per sample between the test set set $\eta$ and the training set as a measure of DE for each probe $i$
\begin{equation}
\Delta LD(i,\eta)=\frac{1}{p}\Bigl[ \log E_{\eta}(|Z_{i}|^{p})-\log E_{\bar{\eta}}(|Z_{i}|^{p})\Bigr]~,~p\in(-1,\alpha)\nonumber
\end{equation}
Here, $E_{\eta}(.)$ is the expectation value for population $\eta$ (test set) and $E_{\bar{\eta}}(.)$ for the rest (training set).  Note that Lemma~\ref{FLOM} guarantees that RHS of the previous equation is indeed independent of $p$. For the calculation reported here $p=\alpha/1.01$, just as it was during structure learning. Thus, $\Delta LD(i,\eta)$ measures the average increase (or decrease) in log-dispersion for the noise variable $Z_i$ corresponding to probe $i$ within population $\eta$. This density is represented as a heat map in Figure~\ref{fig: CV}C. We should point out that a higher (or lower) dispersion for the noise variable associated with a gene in the test set does not necessarily imply over (or under) expression  of a gene in the test set population. The change in dispersion could also be due to a change in network topology or regression coefficients for the test set population.

\section{Discussion}
In this paper we have introduced and developed the theory for efficiently learning $\alpha$-SG models from data. In particular, one of the main contributions of this paper is to show how the BIC can be asymptotically reduced to the MDC for $\alpha$-SG models. This result makes it feasible to efficiently learn the structure of these models, since the log-likelihood term does not have a closed form expression in general. We have also empirically validated the resultant algorithm {\tt StabLe} on both simulated and microarray data. In both cases, the presence of heavy-tailed noise has a clear effect on learning performance of OLS based methods. Based on these results, we recommend a bootstrapped estimation of $\alpha$ as an effective and computationally efficient diagnostic to assess the applicability of OLS based Gaussian graphical models. 

We have also described {\tt SGEX}, a new technique for quantifying differential expression from microarray data.  $\alpha$-SG models may also have wider applicability to other aspects of computational biology, especially to data from next-generation sequencing technologies. In addition to mRNA expression measurements (RNA-seq experiments), $\alpha$-SG models may prove helpful for other experiments, such as protein-DNA binding (ChIP-seq experiments) and DNA accessibility measurements (DNase-seq and FAIRE-seq experiments).

Finally, we should mention that there are several potential applications of $\alpha$-SG models beyond computational biology. In particular, image processing provides several problem instances where there is a need to relate different regions of the image. For example, functional magnetic resonance imaging (fMRI) experiments generate a series of images highlighting activity sites in the brain in response to stimuli. Bayesian networks are an effective way of modeling statistical relations between different areas of the brain and the stimuli~\citep{li11}.  Stable distributions may provide a better model for such applications. Another image processing application with potentials for  $\alpha$-SG models is remote sensing images of the earth \citep{mustafa12} where image histograms demonstrate clearly skewed and heavy tailed characteristics \citep{kuruoglu04}. Traffic modeling \citep{castillo12} and financial data analysis~\citep{bonato2012modeling} are also promising application areas.

\section{Software Availability}
Source code for {\tt StabLe} and data sets used here are available at \\{\tt https://sourceforge.net/projects/sgmodels/}.

\begin{thebibliography}{40}
\providecommand{\natexlab}[1]{#1}
\providecommand{\url}[1]{\texttt{#1}}
\expandafter\ifx\csname urlstyle\endcsname\relax
  \providecommand{\doi}[1]{doi: #1}\else
  \providecommand{\doi}{doi: \begingroup \urlstyle{rm}\Url}\fi

\bibitem[Achim and Kuruoglu(2005)]{achim2005image}
A.~Achim and E.~E. Kuruoglu.
\newblock Image denoising using bivariate $\alpha$-stable distributions in the
  complex wavelet domain.
\newblock \emph{IEEE Signal Processing Letters}, 12\penalty0 (1):\penalty0
  17--20, 2005.

\bibitem[Achim et~al.(2001)Achim, Bezerianos, and Tsakalides]{achim2001novel}
A.~Achim, A~Bezerianos, and P.~Tsakalides.
\newblock Novel {B}ayesian multiscale method for speckle removal in medical
  ultrasound images.
\newblock \emph{IEEE Transactions on Medical Imaging}, 20\penalty0
  (8):\penalty0 772--783, 2001.

\bibitem[Ben-Dor et~al.(2000)Ben-Dor, Bruhn, Friedman, Nachman, Schummer, and
  Yakhini]{ben2000tissue}
A.~Ben-Dor, L.~Bruhn, N.~Friedman, I.~Nachman, M.~Schummer, and Z.~Yakhini.
\newblock Tissue classification with gene expression profiles.
\newblock \emph{Journal of Computational Biology}, 7\penalty0 (3-4):\penalty0
  559--583, 2000.

\bibitem[Berger and Mandelbrot(1963)]{Berger63}
J.~Berger and B.~Mandelbrot.
\newblock A new model for error clustering in telephone circuits.
\newblock \emph{IBM Journal of Research and Development}, pages 224--236, 1963.

\bibitem[Bickson and Guestrin(2011)]{Bickson}
D.~Bickson and C.~Guestrin.
\newblock Inference with multivariate heavy-tails in linear models.
\newblock In \emph{Proceedings of NIPS}, 2011.

\bibitem[Bonato(2012)]{bonato2012modeling}
M.~Bonato.
\newblock Modeling fat tails in stock returns: a multivariate stable-{GARCH}
  approach.
\newblock \emph{Computational Statistics}, 27\penalty0 (3):\penalty0 499--521,
  2012.

\bibitem[Byrd and Payne(1979)]{byrd1979convergence}
R.~H. Byrd and D.~A. Payne.
\newblock Convergence of the iteratively reweighted least squares algorithm for
  robust regression.
\newblock Technical Report 313, The Johns Hopkins University, Baltimore, MD,
  1979.

\bibitem[Castillo et~al.(2012)Castillo, Nogal, Men\'{e}ndez,
  S\'{a}nchez-Cambronero, and Jim\'{e}nez]{castillo12}
E.~Castillo, M.~Nogal, M.~Men\'{e}ndez, J., S.~S\'{a}nchez-Cambronero, and
  P.~Jim\'{e}nez.
\newblock Stochastic demand dynamic traffic models using generalized
  beta-{G}aussian {B}ayesian networks.
\newblock \emph{IEEE Transactions on Intelligent Transportation Systems},
  13\penalty0 (2):\penalty0 565--581, 2012.

\bibitem[Chambers et~al.(1976)Chambers, Mallows, and Stuck]{Chambers76}
J.~Chambers, C.~Mallows, and B.~Stuck.
\newblock A method for simulating stable random variables.
\newblock \emph{Journal of the American Statistical Association}, 71\penalty0
  (354):\penalty0 340--344, 1976.

\bibitem[Cooper and Herskovits(1992)]{Cooper}
G.~Cooper and E.~Herskovits.
\newblock A {B}ayesian method for the induction of probabilistic networks from
  data.
\newblock \emph{Machine Learning}, 9:\penalty0 309--347, 1992.

\bibitem[Daubechies et~al.(2010)Daubechies, DeVore, Fornasier, and
  G\"{u}nt\"{u}rk]{Daubechies}
I.~Daubechies, R.~DeVore, M.~Fornasier, and C.~S. G\"{u}nt\"{u}rk.
\newblock Iteratively reweighted least squares minimization for sparse
  recovery.
\newblock \emph{Communications on Pure and Applied Mathematics},
  {LXIII}:\penalty0 1--38, 2010.

\bibitem[Feller(1968)]{feller1968introduction}
W.~Feller.
\newblock \emph{An Introduction to Probability Theory, vol. I, vol. II}.
\newblock John Wiley, New York, 1968.

\bibitem[Friedman(2004)]{friedman2004inferring}
N.~Friedman.
\newblock Inferring cellular networks using probabilistic graphical models.
\newblock \emph{Science}, 303\penalty0 (5659):\penalty0 799--805, 2004.

\bibitem[Friedman et~al.(2000)Friedman, Linial, Nachman, and
  Pe'er]{friedman2000using}
N.~Friedman, M.~Linial, I.~Nachman, and D.~Pe'er.
\newblock Using {B}ayesian networks to analyze expression data.
\newblock \emph{Journal of computational biology}, 7\penalty0 (3-4):\penalty0
  601--620, 2000.

\bibitem[Gallardo et~al.(2000)Gallardo, Makrakis, and
  Orozco-Barbosa]{gallardo2000use}
J.~R. Gallardo, D.~Makrakis, and L.~Orozco-Barbosa.
\newblock Use of $\alpha$-stable self-similar stochastic processes for modeling
  traffic in broadband networks.
\newblock \emph{Performance Evaluation}, 40\penalty0 (1):\penalty0 71--98,
  2000.

\bibitem[Hardin~Jr(1984)]{hardin1984skewed}
C.~D. Hardin~Jr.
\newblock Skewed stable variables and processes.
\newblock Technical Report~79, Univ. North Carolina, Chapel Hill, 1984.

\bibitem[Heckerman et~al.(2000)Heckerman, Chickering, Meek, Rounthwaite, and
  Kadie]{Heckerman00}
D.~Heckerman, D.~Chickering, C.~Meek, R.~Rounthwaite, and C.~Kadie.
\newblock Dependency networks for density estimation, collaborative filtering,
  and data visualization.
\newblock \emph{Journal of Machine Learning Research}, 1:\penalty0 49--75,
  2000.

\bibitem[{International HapMap 3 Consortium} and {others}(2010)]{HapMap3}
{International HapMap 3 Consortium} and {others}.
\newblock Integrating common and rare genetic variation in diverse human
  populations.
\newblock \emph{Nature}, 467\penalty0 (7311):\penalty0 52--58, 2010.

\bibitem[Koller and Friedman(2009)]{Koller}
D.~Koller and N.~Friedman.
\newblock \emph{Probabilistic graphical models: principles and techniques}.
\newblock MIT press, Cambridge, MA, 2009.

\bibitem[Kuruoglu(2001)]{Kuruoglu01}
E.~E. Kuruoglu.
\newblock Density parameter estimation of skewed $\alpha$-stable distributions.
\newblock \emph{IEEE Transactions on Signal Processing}, 49\penalty0
  (10):\penalty0 2192--2201, 2001.

\bibitem[Kuruoglu and Zerubia(2004)]{kuruoglu04}
E.~E. Kuruoglu and J.~Zerubia.
\newblock Modeling {SAR} images with a generalization of the {R}ayleigh
  distribution.
\newblock \emph{IEEE Transactions on Image Processing}, 13\penalty0
  (4):\penalty0 527--533, 2004.

\bibitem[L{\'e}vy(1925)]{levy1925calcul}
P.~L{\'e}vy.
\newblock \emph{Calcul des probabilit{\'e}s}.
\newblock Gauthier-Villars Paris, 1925.

\bibitem[Li et~al.(2011)Li, Chen, Fleisher, Reiman, Yao, and Wu]{li11}
R.~Li, K.~Chen, A.~S. Fleisher, E.~M. Reiman, L.~Yao, and X.~Wu.
\newblock Large-scale directional connections among multi resting-state neural
  networks in human brain: A functional mri and bayesian network modeling
  study.
\newblock \emph{NeuroImage}, 56\penalty0 (3):\penalty0 1035--1042, 2011.

\bibitem[Mandelbrot(1963)]{Mandel63}
B.~Mandelbrot.
\newblock The variation of certain speculative prices.
\newblock \emph{Journal of Business}, 26:\penalty0 394--419, 1963.

\bibitem[Montgomery et~al.(2010)]{montgomery2010transcriptome}
S.~B. Montgomery et~al.
\newblock Transcriptome genetics using second generation sequencing in a
  caucasian population.
\newblock \emph{Nature}, 464\penalty0 (7289):\penalty0 773--777, 2010.

\bibitem[Mustafa et~al.(2012)Mustafa, Tolpekin, and Stein]{mustafa12}
Y.~T. Mustafa, V.~A. Tolpekin, and A.~Stein.
\newblock Application of the expectation maximization algorithm to estimate
  missing values in gaussian bayesian network modeling for forest growth.
\newblock \emph{IEEE Transactions on Geoscience and Remote Sensing},
  50\penalty0 (5):\penalty0 1821--1831, 2012.

\bibitem[Nikias and Shao(1995)]{Nikias}
C.~L. Nikias and M.~Shao.
\newblock \emph{Signal Processing with Alpha-Stable Distributions}.
\newblock Wiley, New York, 1995.

\bibitem[Nolan(2013)]{Nolan13}
J.~P. Nolan.
\newblock \emph{Stable Distributions - Models for Heavy Tailed Data}.
\newblock Birkh\"{a}user, Boston, {C}hapter 1 online at
  academic2.american.edu/~jpnolan edition, 2013.

\bibitem[Nolan and Rajput(1995)]{NolanMulti}
J.~P. Nolan and B.~Rajput.
\newblock Calculation of multi-dimensional stable densities.
\newblock \emph{Communications in Statistics - Simulation and Computation},
  24\penalty0 (3):\penalty0 551--566, 1995.

\bibitem[Pearl(1988)]{Pearl}
J.~Pearl.
\newblock \emph{{P}robabilistic {R}easoning in {I}ntelligent {S}ystems}.
\newblock Morgan Kaufmann, San Mateo, CA, 1988.

\bibitem[Salas-Gonzalez et~al.(2009{\natexlab{a}})Salas-Gonzalez, Kuruoglu, and
  Ruiz]{diego2009modelling}
D.~Salas-Gonzalez, E.~E. Kuruoglu, and D.~P. Ruiz.
\newblock Modelling and assessing differential gene expression using the alpha
  stable distribution.
\newblock \emph{The International Journal of Biostatistics}, 5\penalty0
  (1):\penalty0 1--24, 2009{\natexlab{a}}.

\bibitem[Salas-Gonzalez et~al.(2009{\natexlab{b}})Salas-Gonzalez, Kuruoglu, and
  Ruiz]{salas2009heavy}
D.~Salas-Gonzalez, E.~E. Kuruoglu, and D.~P. Ruiz.
\newblock A heavy-tailed empirical bayes method for replicated microarray data.
\newblock \emph{Computational Statistics {\&} Data Analysis}, 53\penalty0
  (5):\penalty0 1535--1546, 2009{\natexlab{b}}.

\bibitem[Samorodnitsky and Taqqu(1994)]{Samorodnitsky94}
G.~Samorodnitsky and M.~S. Taqqu.
\newblock \emph{Stable Non-Gaussian Random Processes}.
\newblock Chapman and Hall, New York, 1994.

\bibitem[Schmidt et~al.(2007)Schmidt, Niculescu-Mizil, and Murphy]{Schmidt}
M.~Schmidt, A.~Niculescu-Mizil, and K.~Murphy.
\newblock Learning graphical model structure using {L}1-regularization paths.
\newblock In \emph{Proceedings of AAAI}, 2007.

\bibitem[Schwarz(1978)]{Schwarz}
G.~Schwarz.
\newblock Estimating the dimension of a model.
\newblock \emph{Annals of Statistics}, 6:\penalty0 461--464, 1978.

\bibitem[Stranger et~al.(2012)]{stranger2012patterns}
B.~E. Stranger et~al.
\newblock Patterns of cis regulatory variation in diverse human populations.
\newblock \emph{PLoS genetics}, 8\penalty0 (4):\penalty0 e1002639, 2012.

\bibitem[Stuck(1978)]{Stuck}
B.~W. Stuck.
\newblock Minimum error dispersion linear filtering of scalar symmetric stable
  processes.
\newblock \emph{IEEE Transactions on Automatic Control}, 23:\penalty0 507--509,
  1978.

\bibitem[Teyssier and Koller(2005)]{Teyssier}
M.~Teyssier and D.~Koller.
\newblock Ordering-based search: A simple and effective algorithm for learning
  {B}ayesian networks.
\newblock In \emph{Proceedings of Uncertainty in Artificial Intelligence
  (UAI)}, 2005.

\bibitem[Tsamardinos et~al.(2006)Tsamardinos, Brown, and
  Aliferis]{tsamardinos2006max}
I.~Tsamardinos, L.~E. Brown, and C.~F. Aliferis.
\newblock The max-min hill-climbing {B}ayesian network structure learning
  algorithm.
\newblock \emph{Machine learning}, 65\penalty0 (1):\penalty0 31--78, 2006.

\bibitem[Zolotarev(1957)]{Zolotarev57}
V.~M. Zolotarev.
\newblock Mellin-{S}tieltjes transforms in probability theory.
\newblock \emph{Theory Probability Appl}, 2:\penalty0 433--460, 1957.

\end{thebibliography}

\newpage
\appendix
\section*{Appendix A}
In this section we provide the proof for Lemma~\ref{specprod}

\noindent{\bf Lemma~\ref{specprod}}~{\it Every $d$-dimensional distribution with a characteristic function of the form
\begin{equation}
\Phi(q|\alpha,\tilde{\mu},\Lambda)=\prod_{k=1}^{d}\phi(c_{k}^Tq|\alpha,\beta_k,\gamma_k,\mu_k)~~\mathrm{where, }~c_k, q\in\mathbb{R}^d\nonumber
\end{equation}
represents a multivariate stable distribution with a finite spectral measure $\Lambda$.}
\begin{proof}
Assume the following ansatz for the spectral measure $\Lambda$,
\begin{eqnarray}
\Lambda_k&=&\frac{\|c_k\|_2^{\alpha}\gamma_k}{2}\Big((1+\beta_k)\delta(s-\frac{c_k}{\|c_k\|_2}) + (1-\beta_k)\delta(s+\frac{c_k}{\|c_k\|_2})\Big)\nonumber\\
\Lambda(ds)&=&\sum_k\Lambda_k  ds\nonumber
\end{eqnarray}
and location vector $\tilde{\mu}$,
\begin{eqnarray}
\eta_k(c_k|\alpha,\beta_k,\gamma_k,\mu_k)&=&\left\{
\begin{array}{cc} 
\mu_k &\alpha\neq 1\\
\mu_k-\frac{2\beta_k\gamma_k}{\pi}\log\|c_k\|_2 &\alpha=1\end{array}
\right.
\nonumber\\
\tilde{\mu} &=& \sum_{k=1}^{d}\eta_k(c_k|\alpha,\beta_k,\gamma_k,\mu_k)c_k ~\in\mathbb{R}^d\nonumber
\end{eqnarray}
Upon substitution into the parametrization in Definition~\ref{MultiDef} we get
\begin{eqnarray}
\int_{S_d}\psi(s^Tq|\alpha)\Lambda_k  ds &=& \frac{\|c_k\|_2^{\alpha}\gamma_k}{2}\Big((1+\beta_k)\psi(\frac{c_k^Tq}{\|c_k\|_2}|\alpha) +(1-\beta_k)\psi(-\frac{c_k^Tq}{\|c_k\|_2}|\alpha)\Big)	\nonumber\\
&=& \frac{\|c_k\|_2^{\alpha}\gamma_k}{2}\frac{|c_k^Tq|^\alpha}{\|c_k\|_2^\alpha}\Big((1+\beta_k)(1-\imath \mathrm{sign}(c_k^Tq) r(\frac{c_k^Tq}{\|c_k\|_2},\alpha)) \nonumber\\
&+& (1-\beta_k)(1+\imath \mathrm{sign}(c_k^Tq)r(\frac{c_k^Tq}{\|c_k\|_2},\alpha))\Big)	\nonumber\\
\implies \int_{S_d}\psi(s^Tq|\alpha)\Lambda_k . ds &=& \gamma_k|c_k^Tq|^{\alpha}\Big(1-\imath \beta_k\mathrm{sign}(c_k^Tq)r(\frac{c_k^Tq}{\|c_k\|_2},\alpha)\Big)\nonumber\\
&=&\gamma_k|c_k^Tq|^{\alpha}\Big(1-\imath \beta_k\mathrm{sign}(c_k^Tq)r(c_k^Tq,\alpha)\Big)\nonumber\\
&-& \imath \beta_k\gamma_k|c_k^Tq|^\alpha\mathrm{sign}(c_k^Tq)\Big(r(\frac{c_k^Tq}{\|c_k\|_2},\alpha)-r(c_k^Tq,\alpha)\Big) \nonumber
\end{eqnarray}
\begin{eqnarray}
\mathrm{Since, }~r(\frac{c_k^Tq}{\|c_k\|_2},\alpha)-r(c_k^Tq,\alpha)&=&\left\{
\begin{array}{cc} 
0 &\alpha\neq 1\\
\frac{2}{\pi}\log\|c_k\|_2 &\alpha=1\end{array}
\right.
\nonumber\\
\imath \beta_k\gamma_k|c_k^Tq|^\alpha\mathrm{sign}(c_k^Tq)\Big(r(\frac{c_k^Tq}{\|c_k\|_2},\alpha)-r(c_k^Tq,\alpha)\Big)
&=&\left\{
\begin{array}{cc} 
0 &\alpha\neq 1\\
\imath c_k^Tq\Big(\frac{2\beta_k\gamma_k}{\pi}\log\|c_k\|_2 \Big)&\alpha=1\end{array}
\right.
\nonumber\\
&=&\left\{
\begin{array}{cc} 
\imath c_k^Tq(\mu_k -\mu_k) &\alpha\neq 1\\
\imath c_k^Tq\Big(  \mu_k - \mu_k +\frac{2\beta_k\gamma_k}{\pi}\log\|c_k\|_2 \Big)&\alpha=1\end{array}
\right.\nonumber \\
&=&\imath c_k^Tq\Big(\mu_k-\eta_k(c_k|\alpha,\beta_k,\gamma_k,\mu_k)\Big)\nonumber
\end{eqnarray}
\begin{eqnarray}
\implies\int_{S_d}\psi(s^Tq|\alpha)\Lambda_k . ds &=& -\log\phi(c_k^Tq|\alpha,\beta_k,\gamma_k,\mu_k) + \imath \eta_k(c_k|\alpha,\beta_k,\gamma_k,\mu_k)c_k^Tq\nonumber\\
\implies\log\Big(\Phi(q|\alpha,\tilde{\mu},\Lambda)\Big)&=&-\int_{S_d}\psi(s^Tq|\alpha)\Lambda(ds) +\imath\tilde{\mu}q \nonumber\\
&=& -\sum_{k=1}^{d}\int_{S_d}\psi(s^Tq|\alpha)\Lambda_k . ds + \imath\sum_{k=1}^{d}\eta_k(c_k|\alpha,\beta_k,\gamma_k,\mu_k)c_k^Tq\nonumber\\
\implies\log\Big(\Phi(q|\alpha,\tilde{\mu},\Lambda)\Big)&=&\sum_{k=1}^{d}\log\phi(c_k^Tq|\alpha,\beta_k,\gamma_k,\mu_k) \nonumber\\
\implies\Phi(q|\alpha,\tilde{\mu},\Lambda)&=&\prod_{k=1}^{d}\phi(c_{k}^Tq|\alpha,\beta_k,\gamma_k,\mu_k)\nonumber
\end{eqnarray}
\end{proof}
\clearpage

\section*{Appendix B}
\subsection*{The \texttt{BARLEY} network }
\begin{figure}[htb!]
\begin{center}
\includegraphics[scale=.44, angle=-90]{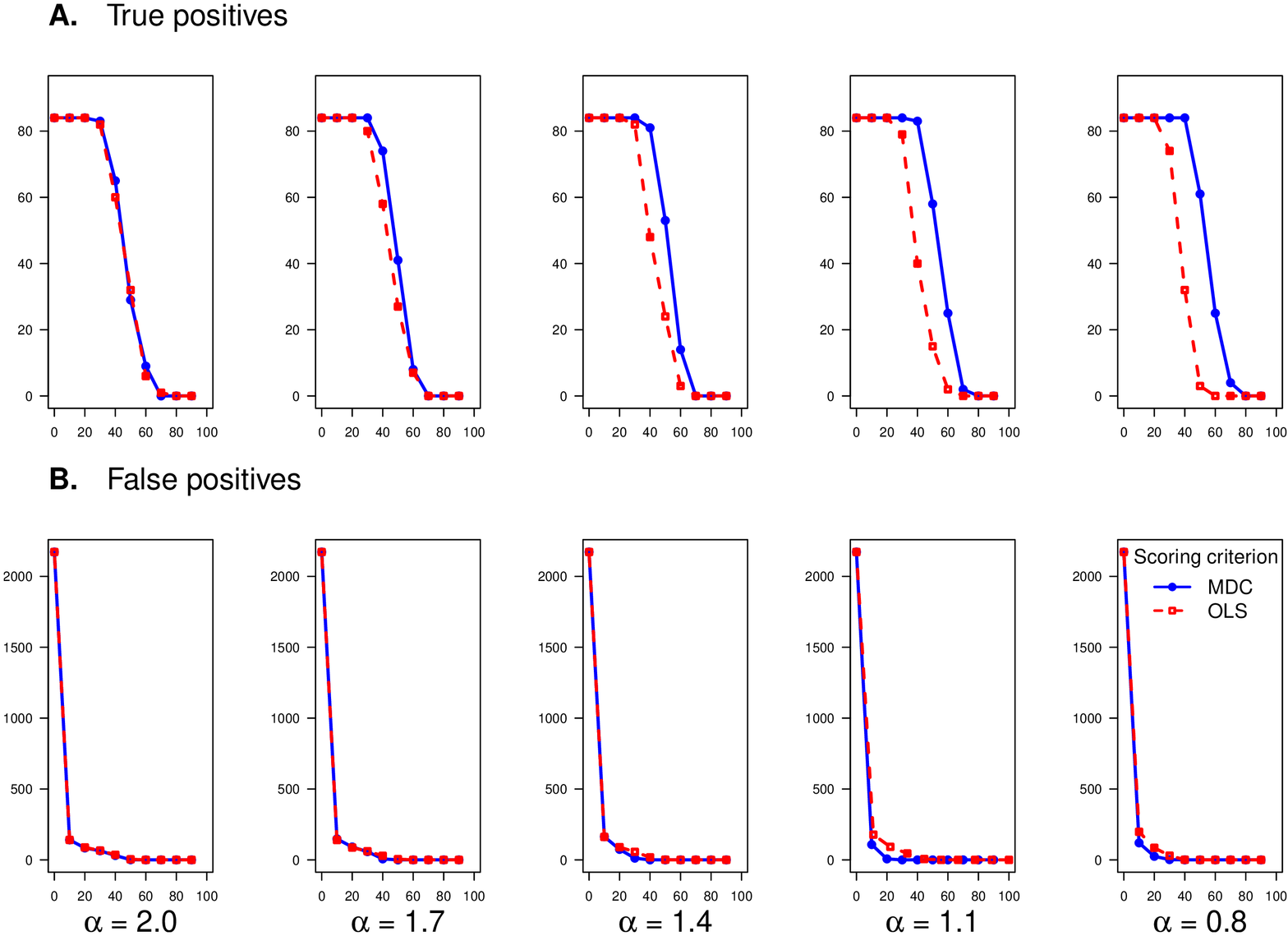}
\end{center}\caption{The \texttt{BARLEY} network - Inferred structure}\label{fig: BARLEY_TPFP}
\end{figure}
\begin{figure}[htb!]
\begin{center}
\includegraphics[scale=.4, angle=-90]{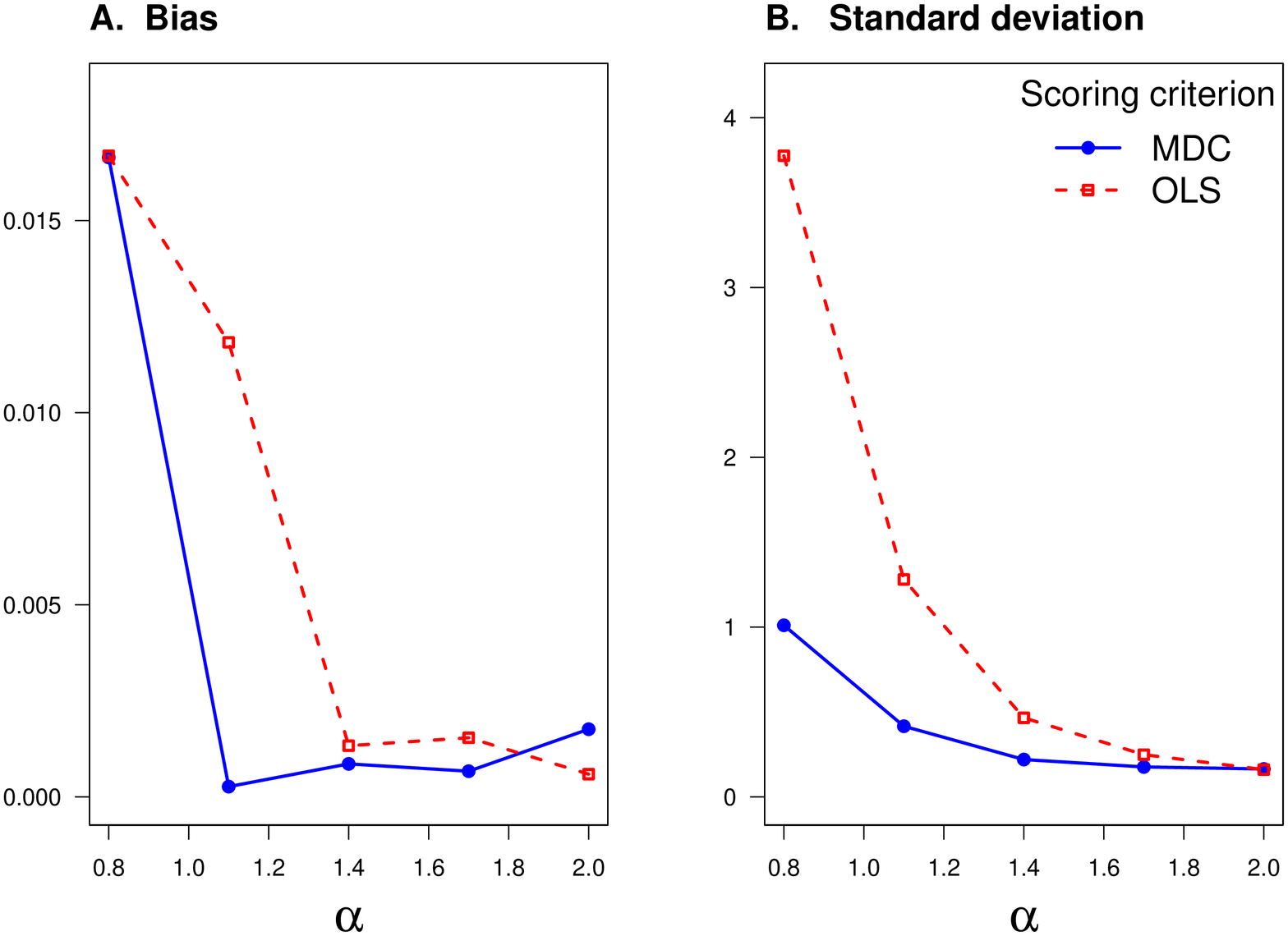}
\end{center}
\caption{The \texttt{BARLEY} network - Estimated regression parameters.}\label{fig: BARLEY_Reg}
\end{figure}
\begin{figure}[b!]
\begin{center}
\includegraphics[scale=.56, angle=-90]{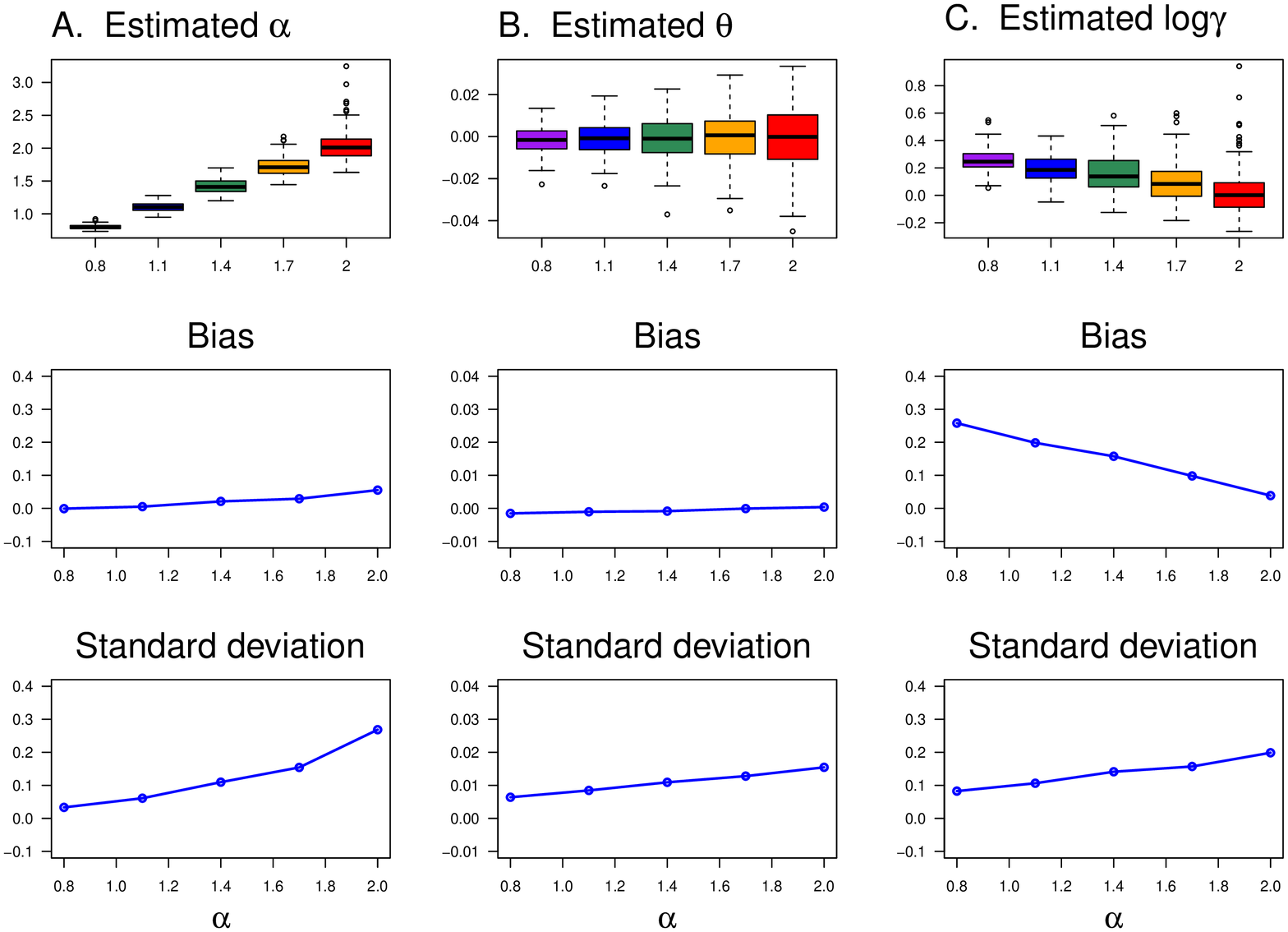}
\end{center}\caption{The \texttt{BARLEY} network - Estimated noise parameters}\label{fig: BARLEY_ABG}
\end{figure}

\clearpage

\subsection*{The \texttt{CHILD network }}

\begin{figure}[htb!]
\begin{center}
\includegraphics[scale=.45, angle=-90]{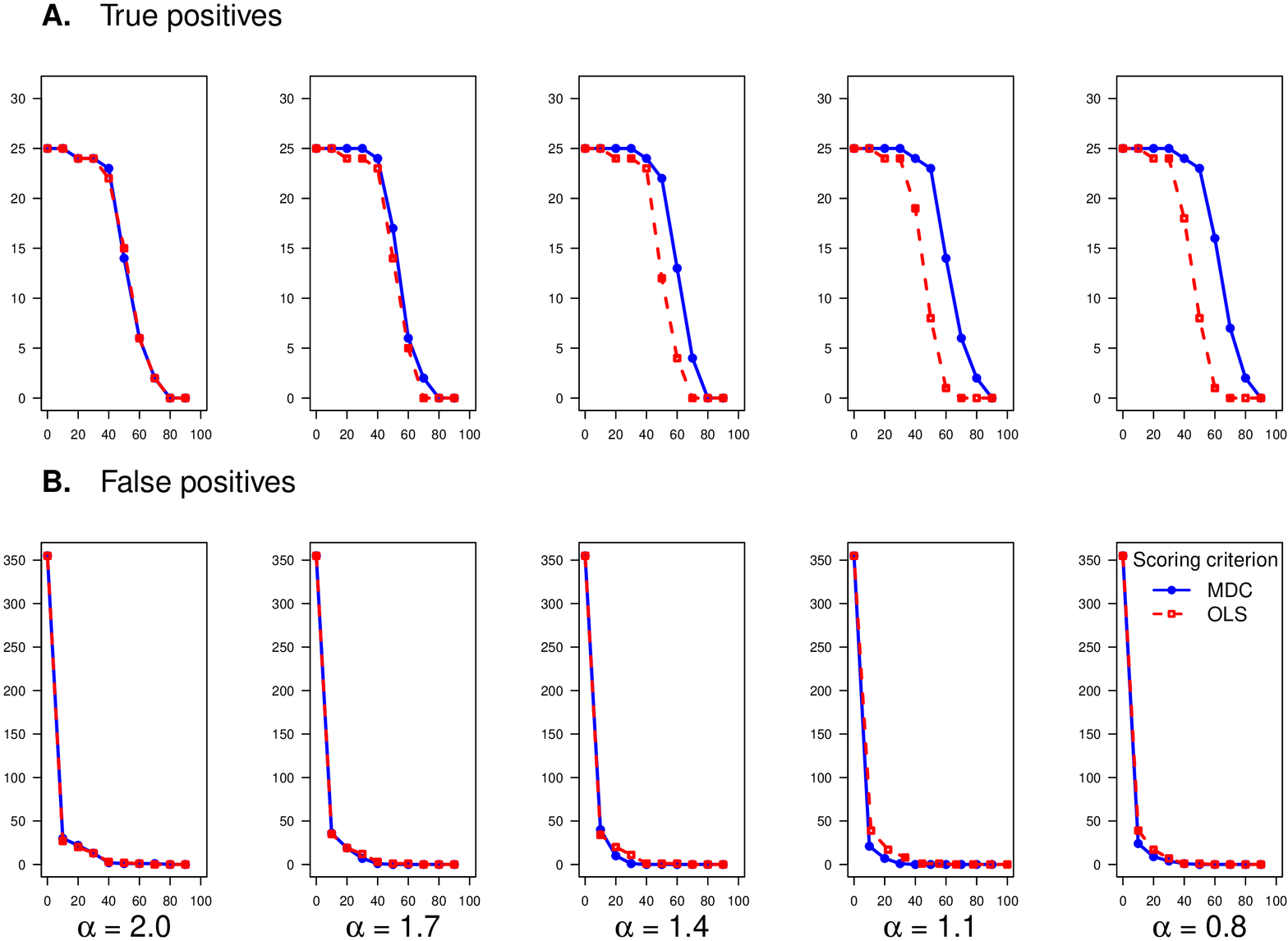}
\end{center}
\caption{The \texttt{CHILD} network - Inferred structure}\label{fig: CHILD_TPFP}
\end{figure}
\begin{figure}[htb!]
\begin{center}
\includegraphics[scale=.4, angle=-90]{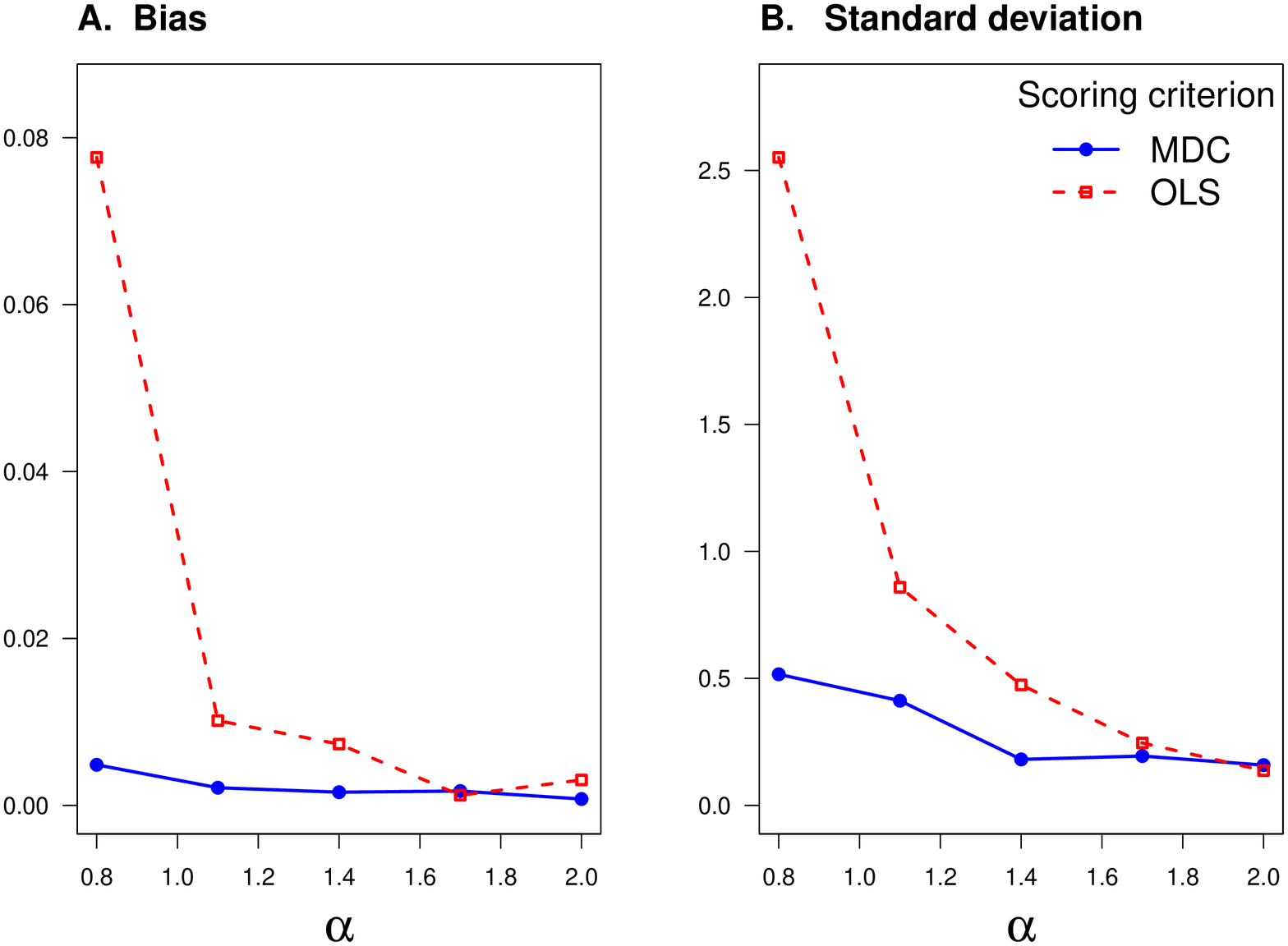}
\end{center}
\caption{The \texttt{CHILD} network - Estimated regression parameters.}\label{fig: CHILD_Reg}
\end{figure}
\begin{figure}[htb!]
\begin{center}
\includegraphics[scale=.56, angle=-90]{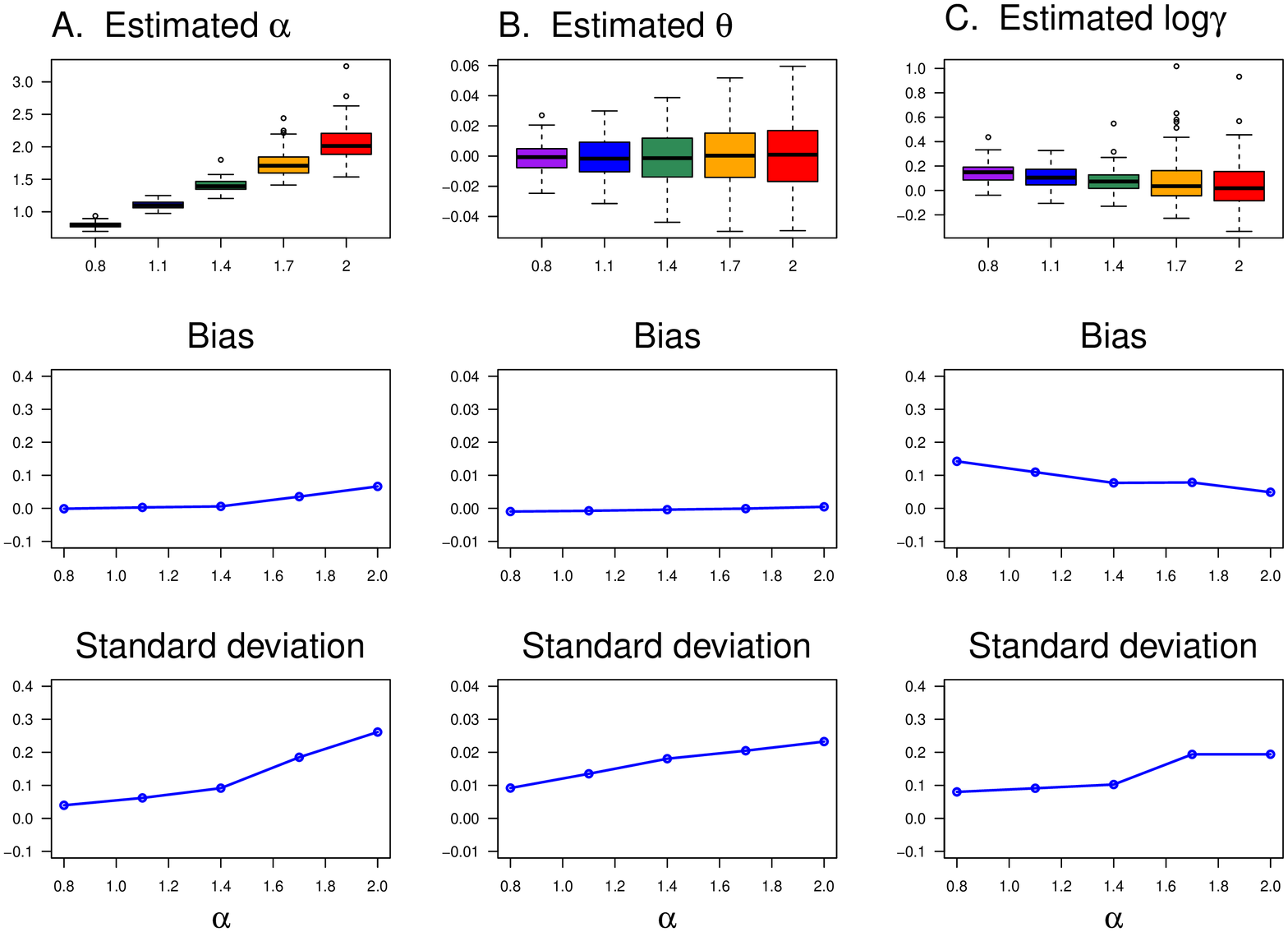}
\end{center}
\caption{The \texttt{CHILD} network - Estimated noise parameters}\label{fig: CHILD_ABG}
\end{figure}
\clearpage
\subsection*{The \texttt{INSURANCE} network }

\begin{figure}[htb!]
\begin{center}
\includegraphics[scale=.45, angle=-90]{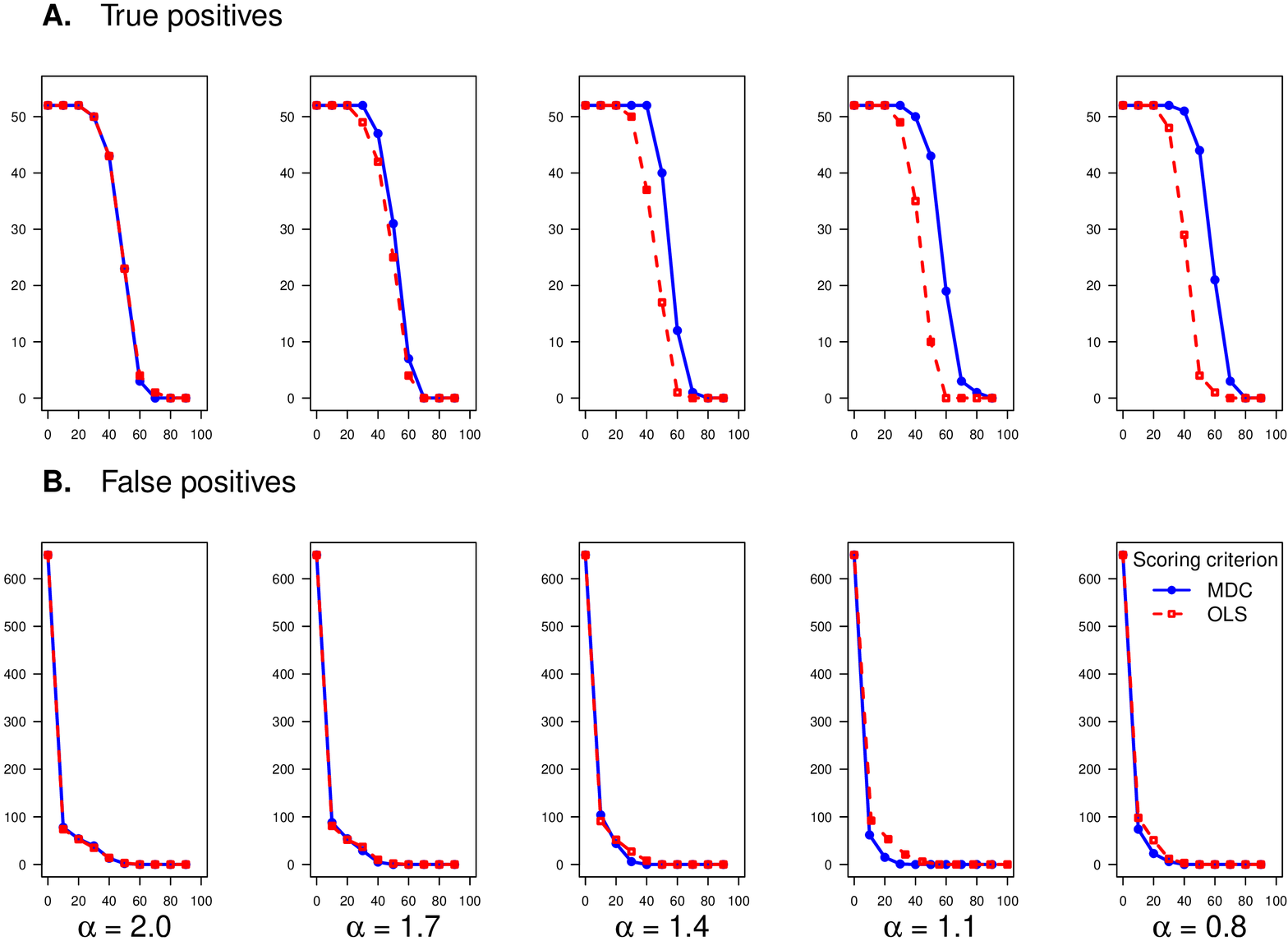}
\end{center}\caption{The \texttt{INSURANCE} network - Inferred structure}\label{fig: INSURANCE_TPFP}
\end{figure}
\begin{figure}[htb!]
\begin{center}
\includegraphics[scale=.4, angle=-90]{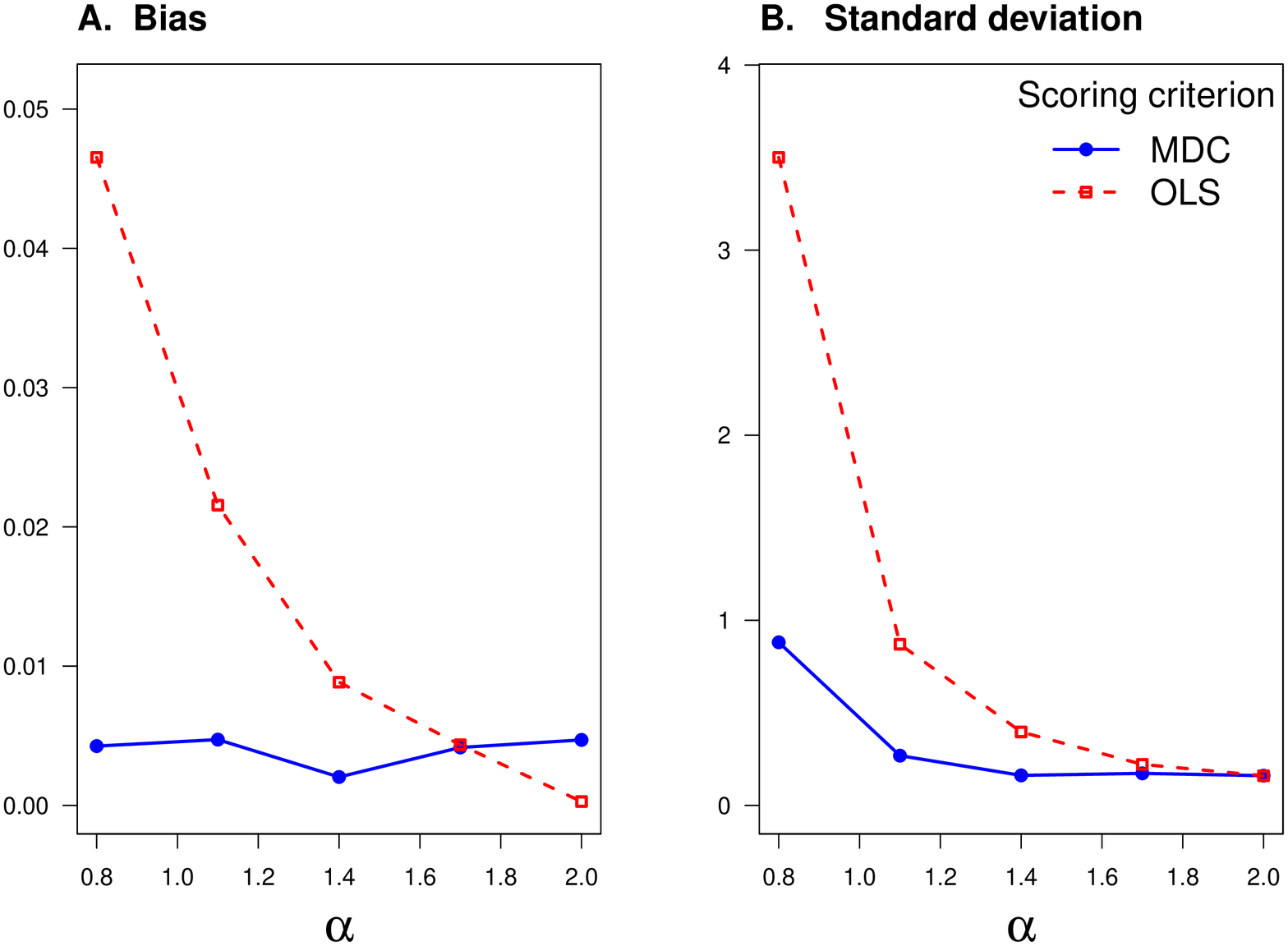}
\end{center}
\caption{The \texttt{INSURANCE} network - Estimated regression parameters.}\label{fig: INSURANCE_Reg}
\end{figure}
\begin{figure}[htb!]
\begin{center}
\includegraphics[scale=.56, angle=-90]{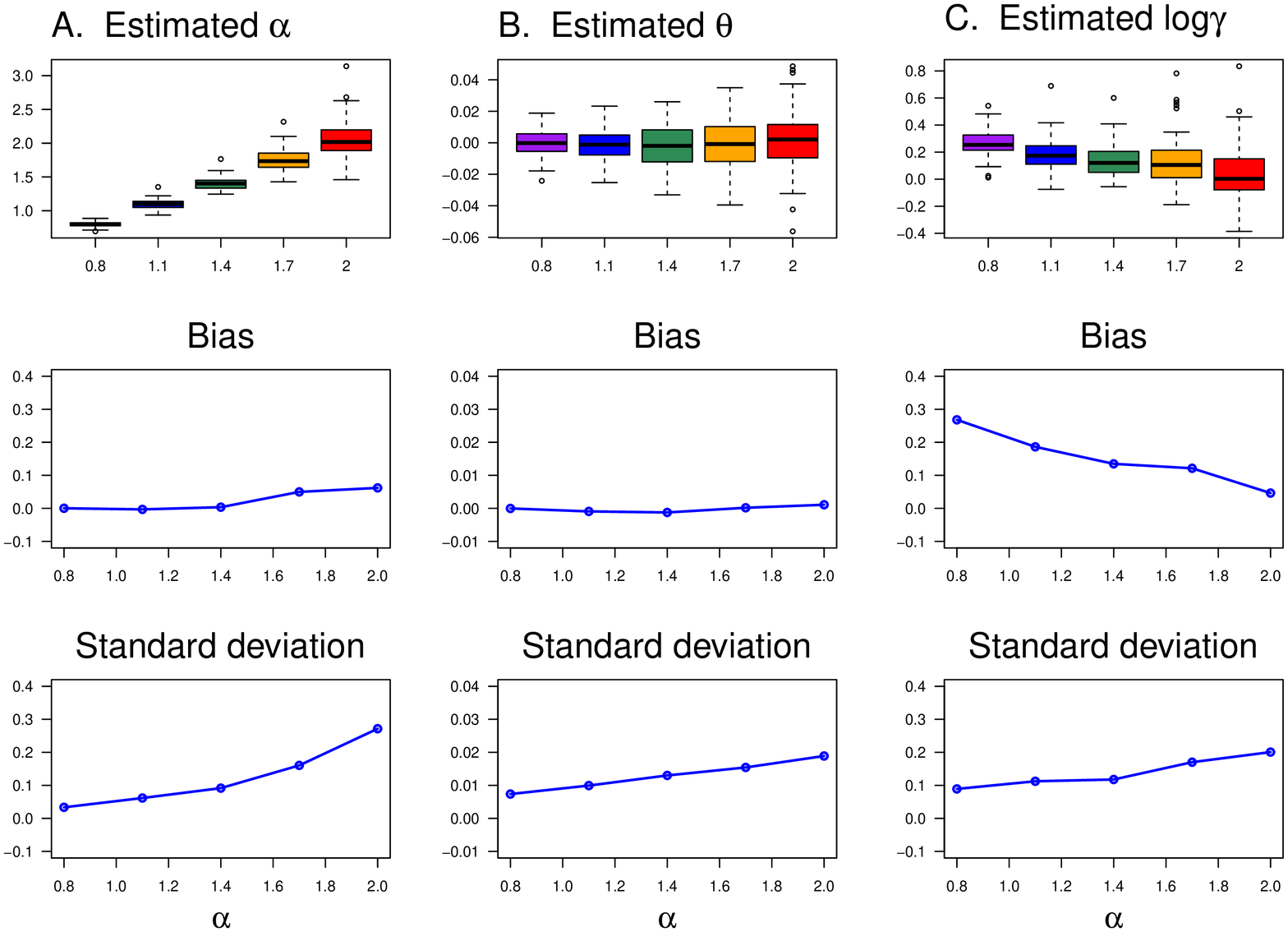}
\end{center}\caption{The \texttt{INSURANCE} network - Estimated noise parameters}\label{fig: INSURANCE_ABG}
\end{figure}
\clearpage

\subsection*{The \texttt{MILDEW} network }

\begin{figure}[htb!]
\begin{center}
\includegraphics[scale=.45, angle=-90]{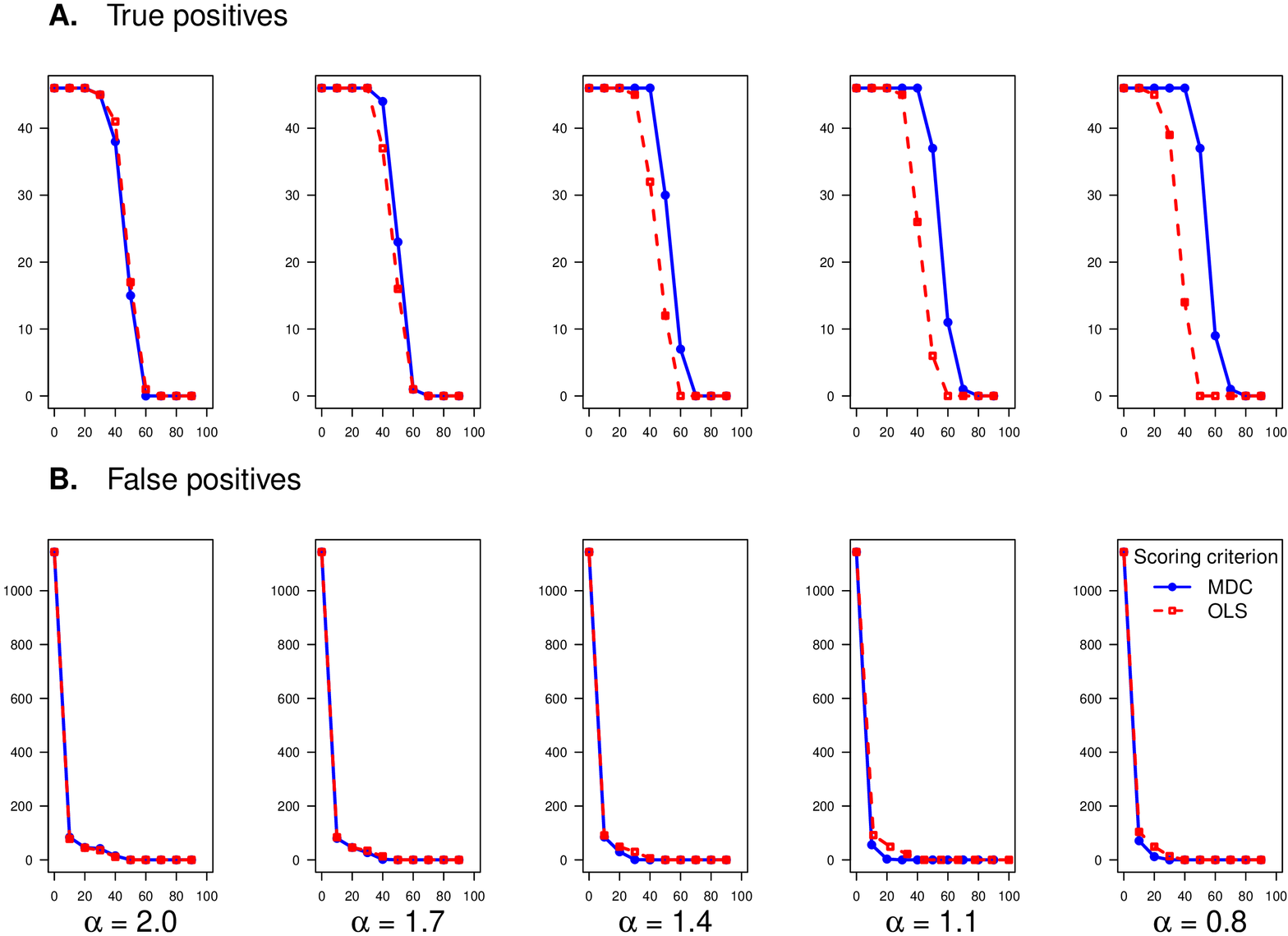}
\end{center}\caption{The \texttt{MILDEW} network - Inferred structure}\label{fig: MILDEW_TPFP}
\end{figure}
\begin{figure}[htb!]
\begin{center}
\includegraphics[scale=.4, angle=-90]{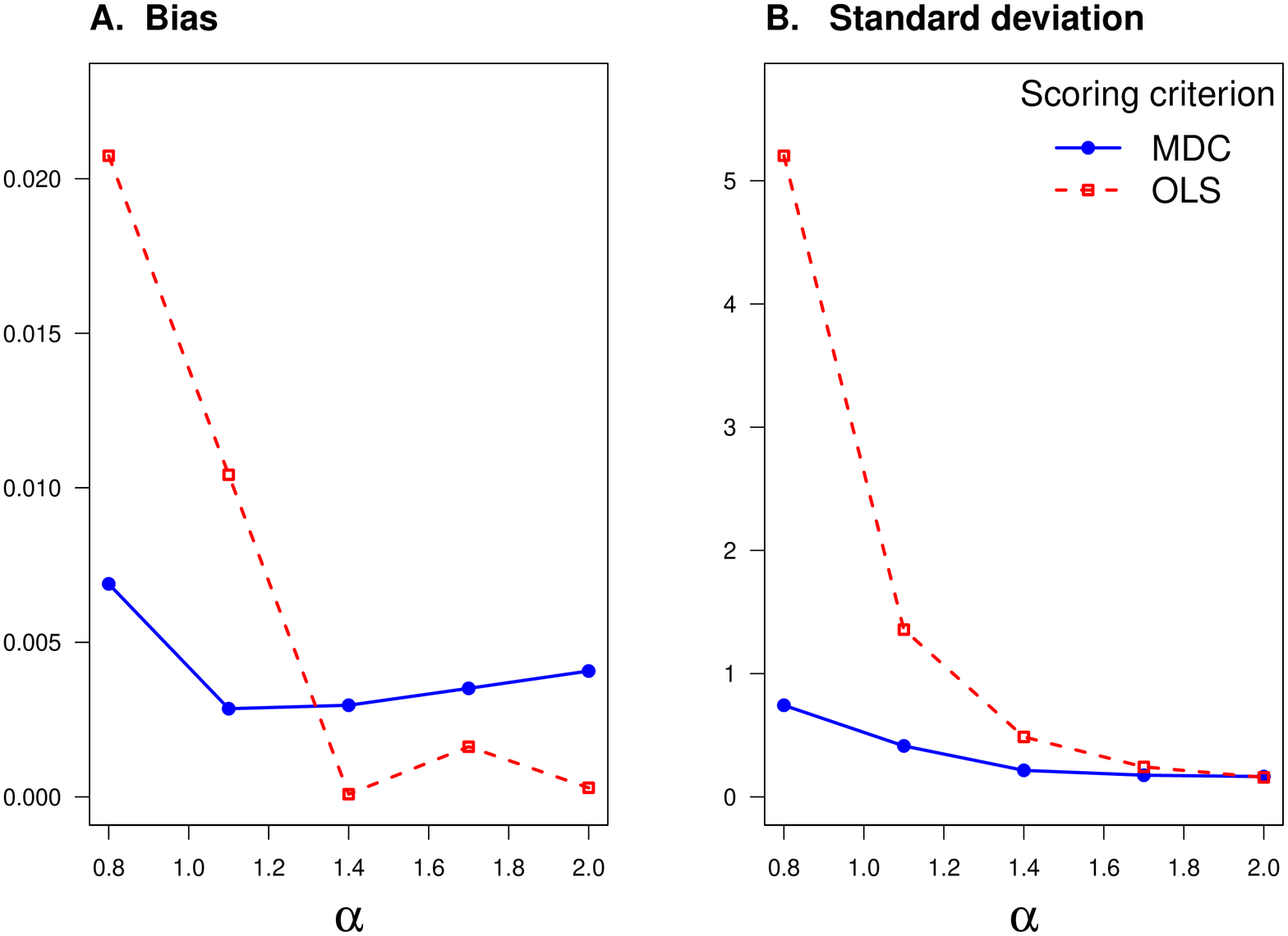}
\end{center}
\caption{The \texttt{MILDEW} network - Estimated regression parameters.}\label{fig:MILDEW_Reg}
\end{figure}
\begin{figure}[htb!]
\begin{center}
\includegraphics[scale=.56, angle=-90]{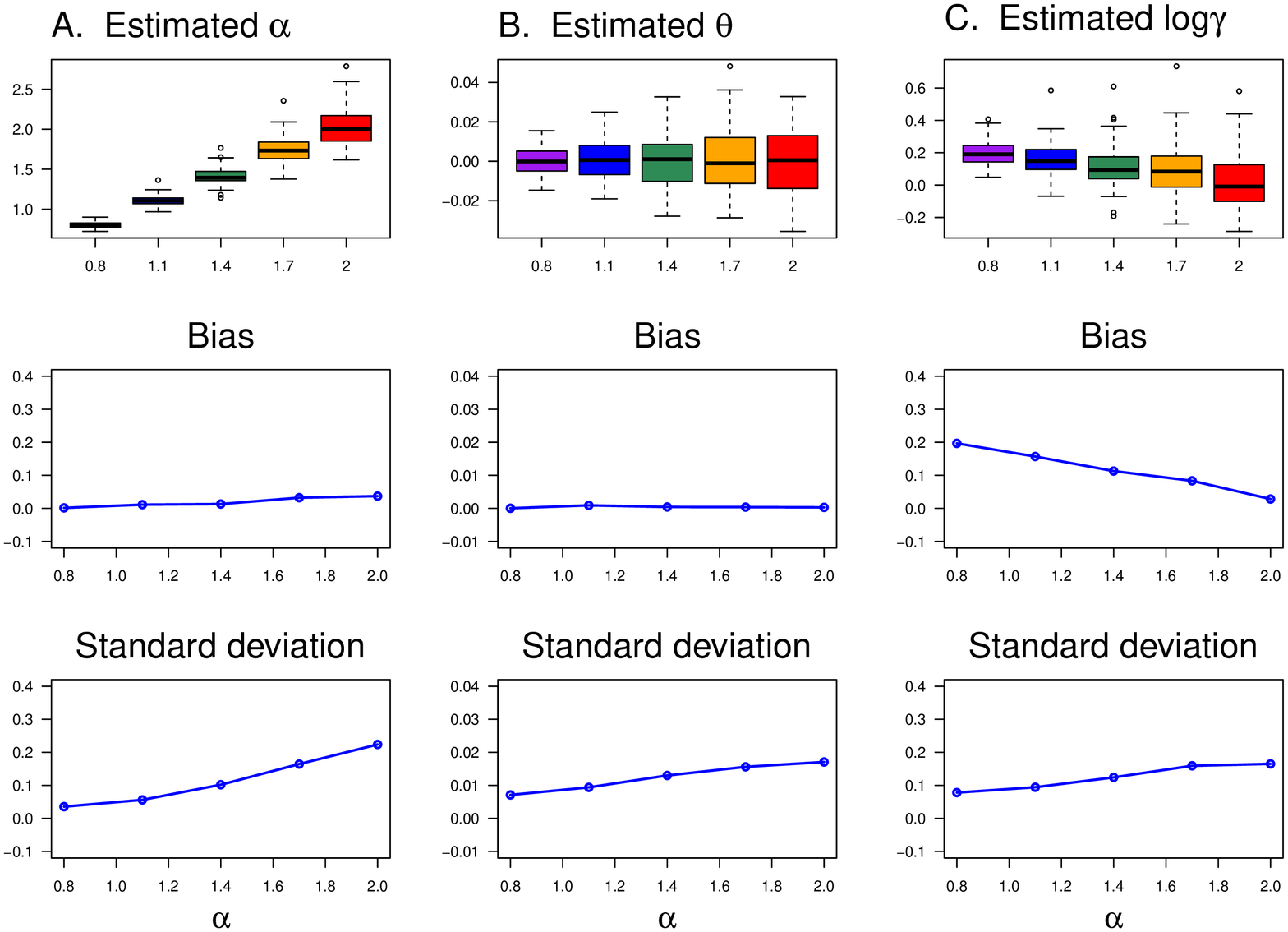}
\end{center}\caption{The \texttt{MILDEW} network - Estimated noise parameters}\label{fig: MILDEW_ABG}
\end{figure}

\end{document}